\documentclass[11pt,letterpaper]{article}

\usepackage{booktabs} 

\usepackage[ruled,linesnumbered,noend,algo2e]{algorithm2e} 

\usepackage{fullpage}

\usepackage{graphicx}
\usepackage{subfigure}
\usepackage{booktabs} 
\usepackage{amsfonts} 
\usepackage{nicefrac} 
\usepackage{enumitem} 
\usepackage{pgfplots} 
	\pgfplotsset{compat=1.14}
\usepackage[square,numbers]{natbib}
\usepackage{amsmath,amssymb,amsthm}
\usepackage{color}
\usepackage{dsfont}
\usepackage{hyperref}
	\hypersetup{bookmarksnumbered}	
\usepackage{authblk}

\DeclareMathOperator*{\argmax}{arg\,max}

\newcommand{\field}[1]{\mathbb{#1}}

\newcommand{\N}{\field{N}}
\newcommand{\E}{\field{E}}

\newcommand{\I}{\field{I}}
\renewcommand{\Pr}{\field{P}}
\renewcommand{\P}{\field{P}}

\newcommand{\cm}{\mathcal{M}}
\newcommand{\scO}{\mathcal{O}}
\newcommand{\co}{\mathcal{O}}
\newcommand{\ck}{\mathcal{K}}
\newcommand{\ct}{\mathcal{T}}
\newcommand{\cv}{\mathcal{V}}

\newcommand{\Ind}[1]{ \field{I}\left\{{#1}\right\} }

\newcommand{\m}{\setminus}
\newcommand{\wh}{\widehat}
\newcommand{\wt}{\widetilde}
\newcommand{\ve}{\varepsilon}
\newcommand{\e}{\varepsilon}

\newcommand{\os}{\mathrm{OS}}

\newcommand{\kl}{\mathrm{KL}}
\renewcommand{\l}{\ldots}
\newcommand{\iop}{\infty}
\newcommand{\iopp}{+\infty}
\newcommand{\phat}{\widehat{p}}
\newcommand{\vhat}{\widehat{v}}

\newtheorem{lemma}{Lemma} 
\newtheorem{theorem}{Theorem} 
\newtheorem{cor}{Corollary}

\title{Dynamic Pricing with Finitely Many Unknown Valuations}

\author[1]{Nicol\`{o} Cesa-Bianchi}
\author[1]{Tommaso R.~Cesari}
\author[2]{Vianney Perchet}

\affil[1]{\small Department of Computer Science \& DSRC, Universit\`a degli Studi di Milano}
\affil[2]{\small CMLA, ENS Paris-Saclay \& CNRS \& Criteo AI Lab, Paris}

\begin{document}

\maketitle


\begin{abstract}
Motivated by posted price auctions where buyers are grouped in an unknown number of latent types characterized by their private values for the good on sale, we investigate revenue maximization in stochastic dynamic pricing when the distribution of buyers' private values is supported on an unknown set of points in $[0,1]$ of unknown cardinality $K$. This setting can be viewed as an instance of a stochastic $K$-armed bandit problem where the location of the arms (i.e., the $K$ unknown valuations) must be learned as well.
\begin{enumerate}
\item In the distribution-free case, we show that our setting is just as hard as $K$-armed stochastic bandits: we prove that no algorithm can achieve a regret significantly better than $\sqrt{KT}$, (where $T$ is the time horizon) and present an efficient algorithm matching this lower bound up to logarithmic factors. 
\item In the distribution-dependent case, we show that for all $K > 2$ our setting is strictly harder than $K$-armed stochastic bandits by proving that it is impossible to obtain regret bounds that grow logarithmically in time or slower. On the other hand, when a lower bound $\gamma > 0$ on the smallest drop in the demand curve is known, we prove an upper bound on the regret of order $\big(1/\Delta + (\log\log T)/\gamma^2\big)\big(K\log T\big)$, where $\Delta$ is the gap between the revenue of the optimal valuation and that of the second-best valuation. 
This is a significant improvement on previously known regret bounds for discontinuous demand curves, that are at best of order $\big(K^{12}/\gamma^8\big) \sqrt{T}$.
\item When $K=2$ in the distribution-dependent case, the hardness of our setting reduces to that of a stochastic $2$-armed bandit: we prove that an upper bound of order $(\log T)/\Delta$ (up to $\log\log$ factors) on the regret can be achieved with no information on the demand curve.
\item Finally, we show a $\scO(\sqrt{T})$ upper bound on the regret for the setting in which the buyers' decisions are nonstochastic, and the regret is measured with respect to the best between two fixed valuations one of which is known to the seller.
\end{enumerate}
\end{abstract}

\section*{Acknowledgements}
\addcontentsline{toc}{section}{Acknowledgements}
Nicol\`o Cesa-Bianchi and Tommaso Cesari gratefully acknowledge the support of Criteo AI Lab through a Faculty Research Award. Vianney Perchet acknowledges the support of the FMJH Program Gaspard Monge in Optimization and operations research (supported in part by EDF) and the CNRS through the PEPS program.


\section{Introduction}
\label{s:intro}
In the online posted price auction problem, also known as dynamic pricing, an unlimited supply of identical goods is sold to a sequence of buyers. To each buyer in the sequence, the seller makes a take-it-or-leave-it offer for the good at a certain price (which we assume to belong to the unit interval $[0,1]$). The good is purchased if and only if the offered price is lower or equal to the buyer's private valuation (also assumed to be in $[0,1]$). At the end of the transaction, the seller's revenue is either zero (if the good is not sold) or equal to the offered price. The buyer's valuation is never observed. Indeed, the seller only learns a single bit for each auction, i.e., whether the good was sold or not at the chosen price. Similarly to previous works \citep{kleinberg2003value, blum2004online, blum2005near}, we assume that the price offered to the $t$-th buyer in the sequence only depends on the past history of observed sales. In particular, we assume that buyers are indistinguishable, and provide no information to the seller other than their willingness to buy at the specified price. For this reason, the seller can post the price for the next buyer publicly, before the buyer shows up.

We evaluate the seller's performance in terms of regret, measuring the difference between the seller's revenue and the revenue achievable by consistently posting the optimal price. The regret in dynamic pricing was initially investigated by \citet{kleinberg2003value} under various assumptions on the generation of the buyers' valuations. In the stochastic setting, in which valuations are drawn i.i.d.\ from a fixed and unknown distribution on $[0,1]$, they show that no algorithm can achieve a $o(\sqrt{T})$ regret and provide an algorithm achieving regret of order $C \sqrt{T\log T}$, where $T$ is the number of buyers in the sequence and $C$ only depends on the distribution of buyers' valuations. Their upper-bound holds under some assumptions on the demand curve, which is the function $D$ mapping each price $x$ to the probability $D(x) = \Pr(V \ge x)$ that the good is sold. Specifically, the revenue function $x \mapsto xD(x)$ is required to have a unique global maximum $x^{\star} \in (0,1)$ and be twice differentiable with a negative second derivative at $x^{\star}$. Without these assumptions, the authors prove a much higher lower bound of order $T^{2/3}$ on the regret. The algorithm achieving the $C\sqrt{T\log T}$ regret under the above assumptions on the demand curve is simple: it runs the UCB1 policy for stochastic bandits~\citep{auer2002finite} on a discretized set of $K = (T/\log T)^{1/4}$ prices.

In this paper, we study the stochastic setting of dynamic pricing under completely different assumptions on the demand curve. Namely, that the distribution of buyers' valuations is supported on an \emph{unknown} set of \emph{unknown} finite cardinality $K$. This models any setting in which buyers are grouped in an unknown number of latent types, characterized by their private values for the good on sale. In particular, this applies to regret minimization in sellers' repeated second-price auctions with a single relevant buyer. This scenario emerges naturally when a seller and a buyer interact repeatedly, and the valuation of the good depends on contextual information known only to the buyer. For instance, in online advertising each time a user lands on a publisher's website, an impression is put on sale to a set of relevant advertisers through an auction (note that whenever there is a single relevant advertiser for the impression, second-price auctions with reserve price are equivalent to posted price auctions). Now, typically, the advertiser's valuation for the impression depends on which segment the user belongs to, where the finite segmentation is based on private information not accessible to the publisher.

Note that our model is very different from assuming that the seller is restricted to offer prices from a \emph{known} finite set of size $K$ \citep{rothschild1974two}, which makes dynamic pricing a special case of $K$-armed stochastic bandits. In our model, the seller does not know the $K$ buyers' valuations, not even their number! So, besides learning which valuation has the highest revenue, the seller must also learn the location of these values. This interplay between noisy search and bandit allocation is one of the main themes of our work.

In contrast with previous approaches, which typically assume parametric \citep{broder2012dynamic} or locally smooth \citep{kleinberg2003value} demand curves, our model with finitely many valuations is equivalent to assuming that the demand curve is piecewise constant with a finite number of discontinuities. Recently, \citet{den2017discontinuous} designed an algorithm for piecewise continuous demand curves achieving an upper bound of order $C \sqrt{T} \log T$ in the piecewise constant case. However, up to constant factors, their hefty leading constant $C$ is at least as big as the maximum between $K^{22} \gamma^{-16} c^{-2}$ and $K^{12} \gamma^{-8} c^{-18}$, where $c$ is the minimum distance between valuations and both $K$ and the smallest drop $\gamma$ in the demand curve must be known in advance. Although their setting extends ours to certain piecewise \emph{parametric} demand curves, we believe that discontinuities are the real source of additional hardness of this dynamic pricing model with respect to previously studied settings.

Our first result is a lower bound of order $\sqrt{KT}$ on the regret in the distribution-free case (where the regret is maximized over all possible demand curves), which holds even when the seller knows the number and position of buyers' private values in advance. This essentially establishes that our setting is at least as hard as a $K$-armed bandit problem. Although we build on the stochastic lower bound of \citet{kleinberg2003value}, our proof is not a simple adaptation of theirs. Indeed, we show that their proof breaks down when $K$ is constant and $T$ grows, which is exactly the regime we are interested in. Then, we present an efficient algorithm achieving a distribution-free upper bound on the regret of order $\sqrt{KT\log T}$ without any additional knowledge of the parameters of the problem.\footnote{Throughout this paper we assume that the time horizon $T$ is known by the seller in advance. This assumption can be easily removed with a ``doubling trick'' (see, e.g., \citep{cesa2006prediction}), a standard technique for extending regret bounds to time sequences of unknown length. 
}
The detailed version of our bound has a significantly better dependence than \citet{den2017discontinuous} on the smallest difference $c$ between two adjacent valuations, and matches---up to logarithmic factors---the lower bound stated above.

In the distribution-dependent case, when the gap $\Delta$ between the revenue of the optimal valuation and that of the second-best valuation is constant, we prove the impossibility of obtaining regret bounds of order significantly better than $\sqrt{T}$ even when $K=3$, thus showing that this setting is strictly harder than $K$-armed stochastic bandits. Motivated by this impossibility result, we investigate distribution-dependent bounds that rely on additional information about the demand curve. By combining suitable generalizations of UCB1 \citep{auer2002finite} and the ``cautious search'' strategy of \citet{kleinberg2003value}, we obtain an efficient algorithm achieving a regret of order at most $\big(1/\Delta + (\log\log T)/\gamma^2\big)\big(K\log T\big)$, where, as before, $\gamma$ is the smallest drop in the demand curve. Since $(K/\Delta)\log T$ is the regret of $K$-armed stochastic bandits, this shows that the price of identifying each one of the $K$ valuations is at most $(\log T)(\log\log T)/\gamma^2$, which corresponds (up to $\log\log$ factors) to the known upper bounds for noisy binary search \citep{karp2007noisy}. 
We conclude the study of the distribution-dependent case by presenting an efficient algorithm with regret of order $(1/\Delta + \log\log T)\log T$ when the number of valuations is known to be at most two. Surprisingly, this bound is the same (up to $\log \log$ terms) as the best possible bound for two-armed stochastic bandits, achievable when not only the number, but also the locations of the valuations are known in advance. In order to prove this result we introduce a novel technique for estimating (up to a multiplicative constant) the expectation $\mu$ of any $[0,1]$-valued random variable with probability at least $1-\delta$, using at most $\co \big(\frac{1}{\mu}\ln\frac{1}{\delta}\big)$ samples, even if the expectation $\mu$ is \emph{not} known in advance. We believe this technique may be valuable in its own right.
\section{Further related works}
\label{s:related}
The literature on dynamic pricing and online posted price auctions is vast. We address the reader to the excellent survey published by \citet{den2015dynamic}, providing a comprehensive picture of the state of the art until the end of 2014 ---see also the tutorial slides by \citet{SZ15} for a perspective more focused on computer science approaches.
An important line of work in dynamic pricing considers a nonstochastic setting in which the sequence of the buyers' private values is deterministic and unknown, and the seller competes against the best fixed price. This model was pioneered by \citet{kleinberg2003value}, who proved a $\scO(T^{2/3})$ upper bound (ignoring logarithmic factors) on the aforementioned notion of regret. Later works \citep{blum2004online,blum2005near} show simultaneous multiplicative and additive bounds on the regret when prices have range $[1,h]$. These bounds have the form $\ve\,G_T^{\star} + \scO\big((h\ln h)/\ve^2\big)$ ignoring $\ln\ln h$ factors, where $G_T^{\star}$ is the total revenue of the optimal price $p^{\star}$. Recent improvements on these results are due to \citet{bubeck2017online}, who prove that the additive term can be made $\scO(p^{\star}(\ln h)/\ve^2)$, where the linear scaling is now with respect to the optimal price rather than the maximum price $h$.
Other variants consider settings in which the number of copies of the item to sell is limited \citep{agrawal2014bandits,babaioff2015dynamic,badanidiyuru2013bandits} or settings in which a returning buyer acts strategically in order to maximize his utility in future rounds \citep{amin2013learning,devanur2014perfect}

Finally, although in this work we focus on the seller's side, regret minimization approaches have been recently applied also on the buyer's side, for example in \citep{mcafee2011design,weed2016online}.

\newcommand{\Rcond}{R^{\mathrm{cond}}}
\section{Preliminaries and definitions}
\label{s:prelim}
We assume all valuations $V_t$ belong to a fixed and unknown finite set $\cv = \{v_1,\ldots,v_K\}\subset[0,1]$, with $0=v_{0}\leq v_{1}<\cdots<v_K\leq v_{K+1}=1$. Unless otherwise specified, the sequence $V_1,V_2,\dots$ is assumed to be sampled i.i.d.\ from a fixed and unknown distribution on $\{v_{1},\ldots,v_{K}\}$. Let $p_{i}=\P(V_{1}=v_{i})$ and assume (without loss of generality) that $p_{i}>0$ for all $i\in\{1,\ldots,K\}$. An instance of the posted price problem is then fully specified by the pairs $(v_1,p_1),\dots,(v_K,p_K)$. We assume auctions are implemented according to the following online protocol: for each round $t\in\{1,2,\dots\}$
\begin{enumerate}[topsep = 0pt, parsep = 0pt, itemsep = 0pt]
\item the seller posts a price $X_t \in [0,1]$
\item buyer's valuation $V_t$, hidden from the seller,  is drawn from $\cv$ according to $\{p_1,\l,p_K\}$
\item the seller observes $\Ind{V_t \ge X_t} \in \{0,1\}$ and computes the revenue $r_t(X_t)=X_t \, \Ind{V_t \ge X_t}$
\end{enumerate}
Note that the expected revenue $\E[r_t(x)] = \E\big[x\,\Ind{V_t\ge x}\big]$ is equal to $x\,D(x)$, where
\begin{equation}
\label{eq:demand}
	D(x) = \P(V_1 \ge x) = \sum_{k \colon v_k \ge x} p_k
\end{equation}
is the \emph{demand curve}. Hence the price maximizing the expected revenue $\E[r_t(x)]$ belongs to the set of valuations $\{v_1,\ldots,v_K\}$ and we denote one of the possible optimal valuations by $v^{\star} = v_{i^{\star}}$. We define the suboptimality gap of $v_j$ with respect to $v^{\star}$ by $\Delta_{j}=\E\big[r_1(v^\star)-r_1(v_j)\big]$. The goal of the seller is to minimize the \emph{regret} 
\[
	R_T = \max_{x \in [0,1]}\E\left[\sum_{t=1}^{T}r_{t}(x)-\sum_{t=1}^{T}r_{t}(X_{t})\right]=\E\left[\sum_{t=1}^{T}r_t(v^{\star})-r_{t}(X_{t})\right] 
\]
where the expectation is understood with respect to any randomness in the generation of $V_1,\dots,V_T$ and $X_1,\dots,X_T$. Formally, a \emph{deterministic seller} is a sequence of functions $X_1,X_2,\dots$ where $X_t = f_t(X_1, Z_1,\dots, X_{t-1}, Z_{t-1})$ is the price posted at time $t$, the random variable $Z_s$ is the binary feedback $\Ind{V_s \ge X_s}$ received by the seller in at time $s$, and $f_t \colon \big( [0,1] \times \{0,1\} \big)^{t-1} \to [0,1]$ is an arbitrary function. A \emph{randomized seller} is a probability distribution over deterministic sellers.

\section{Lower bounds}
\label{s:lower}
In this section we show some important similarities and differences between dynamic pricing with $K$ valuations and the $K$-armed bandit problem. First, we state that in the distribution-free case the former is at least as difficult as the latter. More precisely, if $T \ge K^3$, no algorithm can have regret better than $\sqrt{KT}$ on dynamic pricing with $K$ valuations. The proof of the following theorem is deferred to Appendix~\ref{s:lb-deferred}.
\begin{theorem}
\label{th:kl-adapted}
For any number of valuations $K\ge 3$ and all time horizons $T \ge K^3$ there exist $K$ pairs $(v_{1},p_{1}),\l,(v_{K},p_{K})$
such that the expected regret of any pricing strategy satisfies 
$
	R_{T}
=
	\Omega\big( \sqrt{KT} \big)
$.
\end{theorem}
Next, we show that in the distribution-dependent case, dynamic pricing is strictly harder than multiarmed bandits. More precisely, even if the suboptimality gap $\Delta$ is constant and $K$ is small, no dynamic pricing algorithm can have regret better than $\sqrt{T}$, whereas the distribution-dependent regret of multiarmed bandits is $\scO(\log T)$.
\begin{theorem}
\label{th:lower1}
If for some constant $c^\star > 0$ a seller algorithm has regret smaller than $c^\star\sqrt{T}$ on any instance of the stochastic dynamic pricing problem with at most three valuations, then there exists an instance with $\Delta = \Theta(1)$ on which the algorithm suffers regret $\Omega(\sqrt{T})$.
\end{theorem}
\begin{proof}
We consider two instances. The first has $\Delta = \frac{1}{4}$ and the second has $\Delta = \scO(1/\sqrt{T})$. We prove that if the algorithm has regret $\scO(\sqrt{T})$ on both instances, then it must have regret $\Omega(\sqrt{T})$ on the first instance. The two instances are defined as follows.
\begin{center}
\addtolength{\tabcolsep}{-2pt}
\begin{tabular}{|l|l|l|}
\multicolumn{3}{c}{\textbf{Instance 1}}
\\ \hline
 $v^{(1)}_1 = 0$ & $D^{(1)}(0) = 1$ & $r^{(1)}(0) = 0$\rule{0pt}{3ex}
\\
$v^{(1)}_2 = \frac{1}{2}$ & $D^{(1)}\big(\frac{1}{2}\big) = \frac{1}{2}$ & $r^{(1)}\big(\frac{1}{2}\big) = \frac{1}{4}$\rule{0pt}{3ex}
\\[1mm] \hline
\end{tabular}
\hspace{0.1cm}
\addtolength{\tabcolsep}{-2pt}
\begin{tabular}{|l|l|l|}
\multicolumn{3}{c}{\textbf{Instance 2}}
\\ \hline
$v^{(2)}_1 = 0$ & $D^{(2)}(0) = 1$ & $r^{(2)}(0) = 0$\rule{0pt}{3ex}
\\
$v^{(2)}_2 = \frac{1-\eta}{2}$ & $D^{(2)}\big(\frac{1-\eta}{2}\big) = \frac{1}{2} + \eta$ & $r^{(2)}\big(\frac{1-\eta}{2}\big) =\frac{1+\eta-2\eta^2}{4}$\rule{0pt}{3ex}
\\
$v^{(2)}_3 = \frac{1}{2}$ & $D^{(2)}\big(\frac{1}{2}\big) = \frac{1}{2}$ & $r^{(2)}\big(\frac{1}{2}\big) = \frac{1}{4}$\rule{0pt}{3ex}
\\[1mm] \hline
\end{tabular}
\end{center}
In Instance~1 the optimal price is $v^{(1)}_2 = \frac{1}{2}$ with revenue $\frac{1}{4}$. In Instance~2 the optimal price is $v^{(2)}_2 =\frac{1-\eta}{2}$ with revenue $\frac{1+\eta-2\eta^2}{4} \geq \frac{1}{4} + \frac{\eta}{8}$ for $\eta \le \frac{1}{4}$. Without loss of generality, we can assume that the seller algorithm only posts prices in the set $\big\{0, \frac{1-\eta}{2}, \frac{1}{2} \big\}$. Let $N_\eta(t)$ be the number of times that the price $\frac{1-\eta}{2}$ is posted and let $\nu^{(i)}_t$ be the law of observed rewards up to time $t$ in Instance~$i\in\{1,2\}$. Since prices $0$ and $\frac{1}{2}$ are uninformative (because demand and revenue do no change across the two instances), it follows from standard calculations that the KL divergence between $\nu^{(1)}_t$ and $\nu^{(2)}_t$ is upper bounded by the KL between two Bernoulli of parameter $\frac{1}{2}$ and $\frac{1}{2} + \eta$ times the expected number of times $v_2$ is chosen under Instance~$1$,
\[
\kl \big(\nu^{(1)}_t \,\Vert\, \nu^{(2)}_t\big) \le \kl\left(\frac{1}{2} \;\Big\Vert\; \frac{1}{2} + \eta \right) \E_1 \big[N_\eta(t)\big] \le 4\eta^2\,\E_1 \big[N_\eta(t) \big] \quad \text{ if } \eta \leq \frac{1}{4}
\]
where $\E_1$ denotes expectation under Instance~$1$. Let $R_T^{(i)}$ be the regret under Instance~$i\in\{1,2\}$. Since $r^{(1)}\big(\frac{1-\eta}{2}\big) = \frac{1-\eta}{2}D^{(1)}\big(\frac{1-\eta}{2}\big) = \frac{1-\eta}{4}$, we have $R_T^{(1)} \geq \frac{\eta}{4} \E_1 \big[N_\eta(T)\big]$. Using the assumption that the seller's algorithm has a regret smaller than $c^{\star}\sqrt{T}$, and adapting an argument of \citet[Proof of Theorem~5]{bubeck2013bounded}, we can write
\[
	\frac{\eta}{4} \frac{T}{4}\exp\Big(-4\eta^2\E_1 \big[N_\eta(T)\big]\Big)
\le
	\max\Big\{R_T^{(1)},R_T^{(2)}\Big\}
\le
	 c^{\star}\sqrt{T}~.
\]
Hence, for $\eta = \frac{32c^{\star}}{\sqrt{T}}$, it must hold that $\E_1\big[N_\eta(t)\big] \ge \frac{\ln 2}{4\eta^2}$, which implies that $R_T^{(1)} \ge \frac{\ln 2}{512c^{\star}}\sqrt{T}$.
\end{proof}
Theorem~\ref{th:lower1} can be extended to the case when $K$ is known to the seller. This can be done by adding an extra valuation $v^{(1)}_3 > v^{(1)}_2$ to Instance~$1$ which has either vanishing probability $p_3$ or vanishing distance $v^{(1)}_3 - v^{(1)}_2$ from $v^{(1)}_2$. (In the latter case the value of $v^{(1)}_3$ depends on the algorithms.) In both cases the seller algorithm is unlikely to detect the presence of this extra valuation, and a slight modified proof of Theorem~\ref{th:lower1} can be applied.

This lower bound shows that $\sqrt{T}$ is best possible in the distribution-dependent case even when $K$ is small and $\Delta$ is a constant. In Section~\ref{s:kval} we show how regret bounds can be substantially better than $\sqrt{T}$ when the learner knows the value of the smallest drop in the demand curve.

\section{Distribution-free bounds}
\label{s:kval-distrfree}
In this section we focus on distribution-free bounds, i.e., bounds that do not depend on the demand curve. The regret bound we prove exceeds the theoretical lower bound stated in Section~\ref{s:lower} by a constant term depending only on the distance between adjacent valuations. 

Our Algorithm~\ref{algo:kval-distrfree} works in two phases: a search phase and a bandit phase. In the search phase a binary search for all ``relevant'' valuations is performed. By the end of this phase, a tight estimate of all such valuations is determined with high probability. During the bandit phase a stochastic bandit algorithm is run on the estimated valuations. As it turns out, this simple scheme is enough to ensure an optimal $\sqrt{KT}$ convergence up to an additive constant independent of the distribution of buyer's valuations. Notably, the algorithm \emph{does not} need to know $K$ in advance. 
\begin{algorithm2e}[t]
\caption{\label{algo:kval-distrfree}}
	\LinesNumbered
	\SetKwInput{kwInit}{Initialization}
	\KwIn{$T\in\mathbb{N}$, $\delta\in(0,1)$.}
	\kwInit{$\ck_1 \gets \{1\}$, $k_1 \gets 1$, $a_1 \gets 0$, $b_1 \gets 1$, $a_0 \gets 0$, $\overline{D}(0)\gets 0$.}

	\For(\tcp*[f]{search phase})
	{
		$m = 1, 2, \dots$
	}
	{
		\uIf
		{
			$\{ j \in \ck_m \mid b_j - a_j >  T^{-1/2} \} \neq \varnothing$\label{state:int_too_big}
		}
		{
			pick $i_m = \min \{ j \in \ck_m \mid b_j - a_j >  T^{-1/2} \}$;\label{state:select_int}
			
			offer price $x_m = (a_{i_m} + b_{i_m})/2$ %
			for $\big\lceil 8 \sqrt{T/ k_m} \ln \delta^{-1} \big\rceil$ rounds;\label{state:post_kvalfree}
			
			\uIf(\tcp*[f]{undershooting})
			{
				$\overline{D}(a_{i_m}) - \overline{D}(x_m) < (k_m / T)^{1/4}/2$\label{state:comp1}
			}
			{
				\uIf(\tcp*[f]{check for fake arms})
				{
					$\overline{D}(x_m) - \overline{D}(b_{i_m}) \ge (k_m / T)^{1/4}/2$\label{state:comp2}
				}
				{
					update $a_{i_m} \gets x_m$, $\ck_{m+1} \gets \ck_m$ %
					and $k_{m+1} \gets k_m$;\label{state:usupd}
				}
				\lElse{update $\ck_{m+1} \gets	\ck_m\m\{i_m\}$ and $k_{m+1} \gets k_m$\label{state:killfake}}
			}
			\uElseIf(\tcp*[f]{overshooting})
			{
				$\overline{D}(a_{i_m}) - \overline{D}(x_m) \ge (k_m / T)^{1/4}/2$\label{state:oversh}
			}
			{
				\uIf(\tcp*[f]{new arms})
				{
					$\mathrm{sign}(a_i-x_m)\big(\overline{D}(a_i)-\overline{D}(x_m)\big) \ge (k_m / T)^{1/4}/2$ %
					for all $i$\label{state:comp3}
				}
				{
					set $a_{k_m +1}\gets x_m$, $b_{ k_m +1} \gets b_{i_m}$, %
					$\ck_{m+1} \gets \ck_m \cup \{k_m +1\}$ and $k_{m+1} \gets k_m +1$;\label{state:new-arm}
				}
				
				update $b_{i_m} \gets x_m$, $\ck_{m+1} \gets \ck_m$ and $k_{m+1} \gets k_m$;
			}
		}
		\lElse{denote the last macrostep by $M$ and \textbf{break}\label{state:end_search_phase}}
	}
	run the UCB1 algorithm on the set of prices $\{a_j\}_{j\in \ck_M}$\label{state:bandit_phase}\tcp*{bandit phase}
\end{algorithm2e}
We call \emph{macrostep} a block of consecutive rounds in which the same price is offered consistently. For each price $x$ we denote by $\overline{D}(x)$ the fraction of accepted offers of $x$ during the last macrostep in which $x$ was offered.
At the beginning of the search phase, our algorithm receives as input the time horizon $T$ and a confidence parameter $\delta$. The algorithm then proceeds in macrosteps of length $\big\lceil 8 \sqrt{T/ k_m} \ln \delta^{-1} \big\rceil$, where $k_m$ is the total number of valuations discovered so far. The goal of the search phase is to approximately locate all \emph{relevant} valuations, that is valuations $v_i$ whose associated probability $p_i$ is at least $\sqrt[4]{K/T}$. 

Initially, all relevant valuations belong to $[a_1,b_1] = [0,1]$. The search proceeds as long as there is at least an interval $i$ containing relevant valuations with length larger than $T^{-1/2}$ (line~\ref{state:int_too_big}). When such an interval $i$ is selected at line~\ref{state:select_int}, a macrostep of binary search is performed and the midpoint price $x_m$ of $[a_i,b_i]$ is offered for $\big\lceil 8 \sqrt{T/ k_m} \ln \delta^{-1} \big\rceil$ rounds (line~\ref{state:post_kvalfree}), thus obtaining an estimate of its demand. If the difference in demands (line~\ref{state:comp1}) is smaller than $(k_m/T)^{-1/4}/2$ no new relevant valuation is detected. Before eliminating the lower half of the interval (line~\ref{state:usupd}), a test designed to detect and remove \emph{fake arms} is performed (line~\ref{state:comp2}). We call fake arm  an interval containing no relevant valuations. Fake arms might be inadvertently allocated when intervals are too wide. In that case, the comparison between two distant points may reveal a large difference in demands due to the presence of several nonrelevant valuations in between. If that happens, the fake arm is removed when the interval becomes small enough (line~\ref{state:killfake}). When no significant difference is detected between the demands, all relevant valuations in $[a_i,b_i]$ remain in $[x_m,b_i]$ with high probability after the update. If, on the other hand, a difference between demands is detected (line~\ref{state:oversh}), two things happen. First, a test is performed to detect possible new relevant valuations (line~\ref{state:comp3}). If a new relevant valuation is spotted, a new interval $[x_m, b_i]$ is allocated. Second, the upper half of the interval $[a_i,b_i]$ is removed. If $[a_i,b_i]$ is split into $[a_i,x_m]$ and $[x_m,b_i]$, all relevant valuations are split between the two intervals. If $[a_i,b_i]$ is simply updated as $[a_i,x_m]$---since no significant difference was detected between the demands at $x_m$ and $b_i$---all relevant valuations in $[a_i,b_i]$ remain in $[a_i,x_m]$ with high probability. 

When all intervals become smaller than $T^{-1/2}$ (line~\ref{state:end_search_phase}), the search phase ends and all intervals $[a_i,b_i]$ are returned. At this point each relevant valuation is contained in one of the intervals with high probability. Therefore the algorithm has now access to $T^{-1/2}$-close approximations of all of them, and the bandit phase begins. In the bandit phase, the algorithm UCB1 \citep{auer2002finite} is run on the set of left endpoints of the intervals (line~\ref{state:bandit_phase}).

\begin{theorem}
\label{th:kval-df}
If Algorithm~\ref{algo:kval-distrfree} is run on an unknown number $K$ of pairs $(v_1,p_1),\dots,(v_K,p_K)$ with input parameter $\delta = T^{-2}$, then its regret satisfies
\[
	R_T = \widetilde{\co} \left( \sqrt{KT} \right) + V(V+1)
\qquad \text{where} \quad
	V = \max_{i\in\{1,\ldots,K\}} \frac{v_k^4}{(v_i-v_{i-1})^5}~.
\]
\end{theorem}
We actually prove a slightly improved bound, in which the constant $V(V+1)$ is replaced by the smaller term $K ({v_K^4}/{v_1^4}) \bigl(1 + ({v_K^4}/{c^4}) \bigr)$, where $c = \min_{i\in\{2,\ldots,K\}}\{v_i-v_{i-1}\}$. To give a frame of reference, previously known upper bounds for discontinuous demand curves \citep{den2017discontinuous} are at best of order
$\big({K^{20}}/{c^{18}}\big) \sqrt{T}$, where $v_1$ is assumed to be bounded away from zero and $K$ needs to be known in advance.
\begin{proof}
We begin by proving that at any time time during the search phase, all intervals $[a_i, b_i]$ satisfy $D(b_i)-D(a_i) \ge T^{-1/4}$ with high probability and with the same probability all valuations $v_j$ not belonging to any of these intervals satisfy $p_j <(K/T)^{1/4}$. For any given price $x \in [0, 1]$ offered during the search phase, Hoeffding's inequality implies $\big\lvert \overline{D}(x)-D(x) \big\rvert \le (\lvert \ck_m \rvert/T)^{1/4}/4$ with probability at least $1-2\delta$. Therefore, if $D(x) - D(y) \geq (\lvert \ck_m \rvert/T)^{1/4}$, then $\overline{D}(x)-\overline{D}(y) \ge (\lvert \ck_m \rvert/T)^{1/4}/2$ with probability at least $1-2\delta$. Moreover, if $D(x)=D(y)$, then $\overline{D}(x)-\overline{D}(y) < (\lvert \ck_m \rvert/T)^{1/4}/2$ with probability at least $1-4\delta$. Since at each macrostep the algorithm performs at most $K+1$ comparisons between $\overline{D}(x)$ and $\overline{D}(y)$ for pairs of points $x,y$ (lines \ref{state:comp1}, \ref{state:comp2}, \ref{state:comp3}), the probability that, for at least one of these comparisons, we have
\begin{equation}
\label{eq:bad_event_k_dfree}
\resizebox{0.94\textwidth}{!} 
{$
\left( 
	\overline{D}(x)-\overline{D}(y) < \sqrt[4]{\frac{k_m}{16 \, T}} \;\wedge\; D(x)-D(y) \ge \sqrt[4]{\frac{k_m}{T}}
\right) 
	\;\text{or}\;
\left( 
	\overline{D}(x)-\overline{D}(y) \ge \sqrt[4]{\frac{k_m}{16 \, T}} \;\wedge\; D(x)=D(y)
\right)
$}
\end{equation}
is at most $4(K+1)\delta$. Thus the probability that the event~(\ref{eq:bad_event_k_dfree}) occurs for at least one comparison in at least one macrostep is at most $4(K+1)M\delta$, where $M \le \sqrt{KT}$. This proves the initial claim. By paying an additional $4(K+1)M\delta T = \co\big(K\sqrt{K/T}\big)$ we can therefore assume that event~(\ref{eq:bad_event_k_dfree}) never occurs. In this case at most $K$ binary searches are performed and ---ignoring constants and logarithmic factors--- the regret increases by at most 
$
	\sum_{k=1}^K \sqrt{{T}/{k}}
\le
	\sqrt{T} \int_0^K x^{-1/2} \mathrm{d}x
= 
	2 \sqrt{KT}
.$
We prove now that if $v_K \notin \bigcup_{j\in\ck_M}[a_j,b_j]$ (which implies $p_K < \sqrt[4]{K/T}$), then it is suboptimal. In order for $v_K$ to be optimal, it would have to have at least a revenue higher than $v_1$. Thus
\[
v_K p_K \ge v_1 \implies \sqrt[4]{K/T} > p_K \ge \frac{v_1}{v_K}
\]
which can only happen if $T<K(v_K / v_1)^4$. By paying an additional $K(v_K / v_1)^4$ term in the regret we can therefore assume that $v_K$ is suboptimal. We show now that all other valuations not belonging to $\bigcup_{j\in\ck_M}[a_j,b_j]$ are also suboptimal. Take any valuation $v_j\notin\bigcup_{i\in\mathcal{K}_M}[a_i,b_i]$ (which again, implies $p_j<\sqrt[4]{K/T}$) strictly smaller than $v_K$. In order for $v_j$ to be optimal, it has to at least be better than $v_1$ and $v_j+1$. If $v_j$ is better than $v_1$ 
\begin{equation}
v_j\sum_{k=j}^K p_k \ge v_1 \implies p_{j+1} \ge \frac{v_1}{v_j}-p_j-\sum_{k=j+2}^K p_k.\label{eq:lower_bound_pp1}
\end{equation}
If $v_j$ is better than $v_{j+1}$ 
\[
v_j \sum_{k=j}^K p_k \ge v_{j+1}\sum_{k=j+1}^K p_k \implies p_j \ge \left(\frac{v_{j+1}}{v_j}-1\right)\sum_{k=j+1}^K p_k=\left(\frac{v_{j+1}}{v_j}-1\right)\left(p_{j+1}+\sum_{k=j+2}^K p_k \right)
\]
and lower bounding $p_{j+1}$ as in (\ref{eq:lower_bound_pp1}) gives 
\[
p_j \ge \left(\frac{v_{j+1}}{v_j}-1\right) \left(\frac{v_1}{v_j}-p_j\right)\implies p_j\ge\frac{v_1}{v_{j+1}}\frac{v_{j+1}-v_j}{v_j}\ge\frac{v_1 c}{v_K^2}
\]
where $c=\min_{i\in\{2,\ldots,K\}}\{v_i-v_{i-1}\}$. Being $p_j<\sqrt[4]{K/T}$
this can only happen if $T < K v_K^8/(v_1 c)^4$. Thus we can assume $v_j$ is suboptimal by paying at most an extra $K v_K^8/(v_1 c)^4$ term in the regret. This proves that $v^\star \in \bigcup_{j\in\ck_M}[a_j,b_j]$. Being $b_j-a_j<T^{-1/2}$, offering $a_j$ rather than any $x\in[a_j,b_j]$ results in an regret increase of at most $\sqrt{T}$. Finally, running the UCB1 algorithm \cite{auer2002finite} for standard stochastic bandits adds another $\widetilde{\co}(\sqrt{k_m T})$ term to the regret, where again $k_m \le K$.
\end{proof}
We now discuss the role that $c$ and $v_1$ play in the dynamic pricing problem. Assume that $p_{i^\star} < \sqrt[4]{K/T}$ but there exist valuations $v_j > v^{\star}$ with $p_j \ge \sqrt[4]{K/T}$, and let $v_k$ be the smallest of such valuations. Arguing as in the proof of Theorem \ref{th:kval-df}, one can prove that $d_k = v_{k}-v^{\star}$ must satisfy
\[
	d_k \le \frac{v_K^2}{v_{1}} K \sqrt[4]{\frac{K}{T}}~.
\]
This means that in principle the optimum valuation $v^\star$ could be hiding in any of the intervals $[v_i - d_i, v_i - c)$, where $v_i$ are all valuations with probabilities $p_i \ge \sqrt[4]{K/T}$. Since these intervals become bigger and bigger as $c$ approaches zero, this behavior foils the attempt of identifying the finite support of the problem instance. The smallest valuation $v_1$ is also a natural parameter of the problem for an entirely different reason. Indeed $v_1$ is not just a valuation, it is the only valuation which is also its own revenue. Assume for example that $v_1=0$ (which makes it always suboptimal). Even if this piece of information is known by the seller, and the problem is reduced to $\{v_2,\l,v_K\}\subset (0,1]$, the reduced problem becomes harder as the ``weights'' $\{p_2,\l,p_K\}$ do not sum to $1$ anymore. The worst case happens when $p_1$ is close to $1$. In this case a considerable amount of samples is needed just to locate any of the remaining valuations, let alone the optimal one, in an online fashion, while accruing regret at each round.


\renewcommand{\iota}{\mu}
\renewcommand{\os}{\mbox{\textsc{os}}}
\newcommand{\Dhat}{\wh{D}}
\newcommand{\Dover}{\overline{D}}
\newcommand{\pmin}{p_{\mathrm{min}}}
\newcommand{\scB}{\mathcal{B}}
\section{Distribution-dependent bounds}
\label{s:kval}
In this section we focus on distribution-dependent bounds, i.e., bounds that are parameterised in terms of the demand curve. Our algorithm ignores the number of valuations, but is given a lower bound $\gamma$ on the smallest probability $\pmin$ of a valuation (i.e., the smallest drop in the demand) ---note that $\gamma \le \pmin$ implies $K \le 1/\gamma$, so we also have an upper bound on the number of valuations. The regret bound we prove exceeds the distribution-dependent regret $(K\ln T)/\Delta$ of standard stochastic bandits by a term of order $K(\ln T)(\ln\ln T)/\gamma^2$. On the other hand, if the number $K$ of valuations (counting only those which are at least $T^{-1}$ apart) is exactly known, it is easy to prove an excess regret bound of order $K((\ln T)/\pmin)^2$ even when $\pmin$ (or a lower bound on it) is unknown: The algorithm performs $\scO(\ln T)$ binary search steps for each one of the $K$ valuations, repeating each step $\scO((\ln T)/\gamma^2)$ times and using a value of $\gamma$ that decreases geometrically until all $K$ valuations are found. A similar argument gives the same regret bound in the case when $K$ not known exactly, but $\gamma \le \pmin$ and $c \le \min_k(p_k-p_{k-1})$ are both known. 

In order to introduce in a clear and concise manner the ideas used to prove our main result, we begin by considering an easier setting in which the feedback is provided by an oracle returning the value of the demand curve $D(X_{t})$ at the posted price $X_t$. This is equivalent to
assuming that the feedback is the expectation $\E\big[r_{t}(X_{t}) \mid V_1,\dots,V_{t-1}\big]=X_{t}D(X_{t})$ rather than the random variable $r_{t}(X_{t})$. This simplified setting allows us to focus on the search of the valuations points, abstracting from the problem of estimating the demand curve. We define a seller algorithm that extends the ``cautious search'' strategy for a single unknown valuation (\cite{kleinberg2003value}, see Algorithm~\ref{alg:cautious_search} in Appendix~\ref{s:shrinking}) to an unknown number of unknown valuations.

Our algorithm (Algorithm~\ref{alg:constant_val_K}) initially looks for a single valuation $v_1$, and then allocates searches for new valuations incrementally. Whenever a new value of the demand curve is observed, providing evidence for the existence of a $i$-th previously unseen valuation, an interval $[a_{i},b_{i}]$ (which we associate with a bandit arm) and a step size $\e_{i}$ are allocated. The interval $[a_{i},b_{i}]$ estimates the smallest valuation $v_i$ contained in it. By construction of the algorithm, $v_i$ is never removed from $[a_{i},b_{i}]$ when the interval shrinks. This implies that the more $[a_{i},b_{i}]$ shrinks, the closer $b_{i}D(a_{i})$ gets to the true revenue $v_iD(v_i)$.

\begin{algorithm2e}[t]
	\SetKwInput{kwInit}{Initialization}
	\KwIn{Time horizon $T\in\mathbb{N}$.}
	\kwInit{set $\kappa_{0}\gets 1$, $a_{1}\gets0$, $b_{1}\gets1$, $n_{1}\gets1$, $\e_{1}\gets1/2$, $D_{1}\gets 1$.}

	\For{$t=1$ \textbf{to} $T$\label{state:for}}{
		set $\kappa_{t}\gets\kappa_{t-1}$\;
		compute $i_{t}\gets\argmax_{i\leq\kappa_{t}}b_{i}D_{i}$\label{state:step}\tcp*{greedy pick}
		\lIf(\tcp*[f]{if $[a_{i_{t}},b_{i_{t}}]$ becomes tiny, play $a_{i_{t}}$ for good}){$b_{i_{t}}-a_{i_{t}}\leq 1/T$}
			{post $a_{i_{t}}$}
		\Else{
			post $X_{t}=a_{i_{t}}+n_{i_{t}}\e_{i_{t}}$ and get feedback $X_{t}D(X_{t})$\label{state:post}\;
			\If(\tcp*[f]{increase prices until surpassing the closest $v_{j}$}){$D(X_{t})=D_{i_{t}}$\label{state:if_still_in_I_j}}{
				\lIf{$X_{t}+\e_{i_{t}}<b_{i_{t}}$\label{state:upd-step}}
					{update $n_{i_{t}}\gets n_{i_{t}}+1$}
				\lElse(\tcp*[f]{shrink the interval}){update $a_{i_{t}}\gets X_{t}$, $n_{i_{t}}\gets1$, $\e_{i_{t}}\gets\e_{i_{t}}^{2}$\label{state:shrink1}}
			}
			\Else{\label{state:else_branch}
				\If(\tcp*[f]{a new valuation is found}){$D(X_{t})\notin\{ D_{1},\l,D_{\kappa_{t}},0\} $\label{state:if_branch}}{
					set $\kappa_{t}\gets\kappa_{t-1}+1$\label{state:kplus1},
						$a_{\kappa_{t}}\gets X_{t}$, $b_{\kappa_{t}}\gets b_{i_{t}}$, $n_{\kappa_{t}}\gets1$, $\e_{\kappa_{t}}\gets\e_{i_{t}}$, $D_{\kappa_{t}}\gets D(X_{t})$\label{state:branch}\;
				}
				update $a_{i_{t}}\gets X_{t}-\e_{i_{t}}$, $b_{i_{t}}\gets X_{t}$, $n_{i_{t}}\gets1$, $\e_{i_{t}}\gets\e_{i_{t}}^{2}$\label{state:shrink2}\tcp*{shrink the interval}
			}
		}
	}\label{state:end_alg_constant_val_K}
\caption{\label{alg:constant_val_K}}
\end{algorithm2e}

The algorithm works by performing cautious searches within each interval. At the beginning, all valuations belong to $[a_1,b_1] = [0,1]$. Whenever an interval is selected (line~\ref{state:step}), a step of cautious search is performed (lines~\ref{state:post}--\ref{state:end_alg_constant_val_K}). During a cautious search in $[a_i,b_i]$ with step size $\e_i$, the sequence of values $X_t = a_i + k\e_i$ for $k\in\{1,2,\dots\}$ is posted until a change is spotted in the demand or $X_t$ gets within $\e_i$ of $b_i$. If the latter happens before a change in the demand is discovered (line~\ref{state:upd-step}), the interval shrinks to $[X_{t},b_{i_{t}}]$ and the step size is refined (line~\ref{state:shrink1}). Note that the shrunken interval contains all valuations that were in $[a_i,b_i]$ because the demand did not change. If a change in the demand is spotted (line~\ref{state:else_branch}), then the interval shrinks to $[X_{t}-\e_i,X_{t}]$ and the step size is reduced (line~\ref{state:shrink2}). 
If the new demand value matches the value of $D(b_i)$ the shrunken interval contains again all valuations that were in $[a_i,b_i]$. If the new demand value does not belong to a known interval (line~\ref{state:if_branch}), then a new interval $[X_{t},b_i]$ is allocated (line~\ref{state:branch}). This process continues until the length of the feasible interval $[a_{j},b_{j}]$ of the arm $j$ with the highest $b_{j}D_{j}$ is less than $1/T$. Then the seller offers the same price $a_{j}$ for all remaining rounds. As time goes by, the number $\kappa$ of discovered valuations grows until possibly reaching the actual number of valuations $K$. Simultaneously, each estimate $b_{i}D_{i}$ converges to the revenue of the smallest valuation in the interval. After enough rounds, picking the interval $i$ with the highest $b_{i}D_{i}$ becomes equivalent to choosing a $1/T$-approximation of an optimal valuation.
Without loss of generality, in the analysis of the algorithm, we assume all valuations $v_1,\dots,v_K$ are at least $1/T$ apart. Let $i_s$ be the index of the arm chosen at time $s$ (line~\ref{state:step}). For any $k=\{1,\dots,K\}$, let $\ct_k \in \big\{ t\leq T \mid v_k\in[a_{i_{t}},b_{i_{t}}] \big\}$. The next lemma states that the steps performed by Algorithm~\ref{alg:constant_val_K} in all the intervals that ever contained $v_k$ are those that a cautious search would have performed if run on the single evaluation $v_k$.
\begin{lemma}
\label{lem:equiv}
Suppose Algorithm~\ref{alg:constant_val_K} is run on $K$ valuations $v_1,\dots,v_K$. Pick $k\in\{1,\dots,K\}$ and $n\in\{1,\dots,|\ct_k|\}$. Let $[0,1] \equiv I_1 \supseteq\cdots\supseteq I_n \equiv [a'_n,b'_n]$ be the sequence of the first $n$ intervals computed by $n$ steps of a cautious search for the single valuation $v_k$ with initial interval $[0,1]$. Then $a'_n \le a_{i_t}$ and $b'_n = b_{i_t}$, where $t$ is the $n$-th smallest value in $\ct_k$. Moreover, the price $X_t$ offered by Algorithm~\ref{alg:constant_val_K} at time $t$ is equal to the $n$-th price offered by the cautious search for the single valuation $v_k$.
\end{lemma}
\begin{proof}
Fix a valuation $v_k$. Let $A$ be Algorithm~\ref{alg:constant_val_K} and $C$ be the cautious search for $v_k$. The proof is by induction on $n$. Since $A$ and $C$ both start with interval $[0,1]$ and price $1/2$ the statement holds for $n=1$. Now let $t$ be the $(n+1)$-st smallest value in $\ct_k$ and let $s$ be the largest value in $\ct_k$ that is smaller than $t$. Let $I_n \equiv [a'_n,b'_n]$ be the $n$-th interval computed by $C$. By induction, $a'_n \le a_{i_s}$, $b'_n = b_{i_s}$, and $X_s$ is offered by both $A$ and $C$. The only interesting case to discuss is when the test at line~\ref{state:if_still_in_I_j} is false. There are two subcases: if the test at line~\ref{state:if_branch} is false, then it must be $X_s > v_k$. In this case $C$ overshoots and the interval is updated exactly in the same way by $C$ and $A$ (see line~\ref{state:shrink2}). If the test at line~\ref{state:if_branch} is true, then it must be $v_i < X_s \le v_k$. This is not an overshoot for $C$, so $I_{n+1} \equiv I_n$. $A$, however, creates a new interval $[a,b]$ ---containing $v_k$--- with $a = X_s$, $b = b_{i_s}$, and unchanged step size $\ve_{i_s}$. The next time $t$ this new interval is selected, the price $X_t$ offered by $A$ is the same as the price offered by $C$ because the step size did not change.
\end{proof}
\begin{theorem}
\label{th:oracle}
If Algorithm \ref{alg:constant_val_K} is run on an unknown number $K$ of pairs $(v_1,p_1),\dots,(v_K,p_K)$, then its regret satisfies
$
R_{T}\leq K(3\ln\ln T+10)
$.
\end{theorem}
\begin{proof}
Intervals are indexed in their order of creation (so that interval $1$ is $[0,1]$), and the $i$-th interval is identified with bandit arm $i$. Note that, at any point during the execution of the algorithm, each valuation belongs to some interval. Any interval is created with at least one valuation in it, and shrinks until it only contains the smallest valuation $v_j$ among those that initially belonged to it. Let $\kappa_T$ be the number of intervals created after $T$ rounds. For $i\in\{1,\dots,\kappa_T\}$, denote by $\iota(i)$ the index $j \in \{1,\dots,K\}$ of the smallest $v_{j}\in[a_{i},b_{i}]$. Now fix any $k$ such that $k = \iota(j)$ (i.e., $v_k$ is the smallest value of the $j$-th interval) for some $j \in\{1,\l, \kappa_T\}$. Note that $j = i_{t_k}$ for some $t_k\in\{1,\l,T\}$ because $k = \iota(j)$ implies that when interval $j$ is created $v_k$ is its smallest valuation. Hence the last selected interval containing $v_k$ must be $j$. Let $T_k = |\ct_k|$ and $t_k = \max\ct_k$. Lemma~\ref{lem:equiv} implies that at time $t_k$ the overall number of cautious steps made for $v_k$ is $T_k$, Lemma~\ref{lem:cautious_shrinking_rate} implies $b_j - a_j \le 2/T_k$ at time $t_k$. Now note that $D_j = D(v_k)$ because $k = \iota(j)$. Since $v^{\star}$ belongs to some $[a_{i^{\star}},b_{i^{\star}}]$, and using $D_{i^{\star}} = D(a_{i^{\star}})$, at time $t_k$ we have
$
	v^{\star}D(v^{\star})
\leq
	b_{i^{\star}}D_{i^{\star}}
\leq
	b_j D_j = b_j D(v_k)
\leq
	\bigl(v_k+(b_j-a_j)\bigr)D(v_k)
$.
Then the above implies $T_k \le 2D(v_k)/\Delta_k$ where $\Delta_k = v^{\star}D(v^{\star}) - v_kD(v_k)$. Lemma~\ref{lem:KL} and Lemma~\ref{lem:equiv} also imply
\begin{equation}
\label{eq:equivbound}
	\sum_{t\in\ct_k}\big(r_t(v_k)-r_t(X_{t})\big) \le 3\ln\ln T_k+8~.
\end{equation}
Noting that $\{1,\dots,T\} \subseteq \ct_1 \cup\ldots\cup \ct_K$, we may write
\begin{align*}
	R_{T}
&=
	\sum_{t=1}^{T}\Bigl(r_t\big(v^{\star}\big)-r_t\big(X_{t}\big)\Bigr)
\le
	\sum_{k=1}^K \sum_{t\in\ct_k} \Bigl(r_t\big(v^{\star}\big)-r_t\big(X_{t}\big)\Bigr)
\\&\leq
	\sum_{k=1}^K \Bigl(T_k v^{\star}D\big(v^{\star}\big)-\big(T_k v_kD(v_k)-(3\ln \ln T_k+8)\bigr)\Bigr)
\\&=
	\sum_{k=1}^K \bigl(T_k\Delta_k+3\ln\ln T_k+8\bigr)
\leq
	\sum_{k=1}^K \bigl(2D(v_k)+3\ln\ln T_k+8\bigr)
\leq
	K(10+3\ln\ln T)
\end{align*}
concluding the proof. 
\end{proof}
Next, we extend Algorithm~\ref{alg:constant_val_K} to account for the fact that the actual feedback at time $t$ is the random variable $r_t(X_t)$ rather than its conditional expectation $X_t D(X_t)$. The main intuition is very simple: in order to estimate $D(x)$ we divide time in blocks (called again \emph{macrosteps}) of equal length, and build an estimate $\Dover(x)$ by posting the same price $x$ within each block. In order to decide which arm $i$ to use in each macrostep, we compute an upper confidence bound $U_i$ on the average demand in the $i$-th interval, and then select the arm attaining the highest of such bounds.

\begin{algorithm2e}[t]
	\SetKwInput{kwInit}{Initialization}
	\KwIn{Time horizon $T\in\mathbb{N}$, confidence parameter $\delta \in (0,1)$.}
	\kwInit{set $\kappa_{0}=1$, $a_{1}\gets0$, $b_{1}\gets1$, $n_{1}\gets1$, $\e_{1}\gets1/2$, $\Dover(a_1) = 1$.}

	\For{$m=1$ \textbf{to} $M_\gamma$}{
		set $\kappa_{m}\gets\kappa_{m-1}$\;
		compute $i_{m}\gets\argmax_{i\leq\kappa_{m}}b_{i}U_{i}$\label{state:stoc-greedy}\tcp*{greedy pick}
		\lIf(\tcp*[f]{if $\left[a_{i_{m}},\,b_{i_{m}}\right]$ gets tiny, play $a_{i_{m}}$ for good}){$b_{i_{m}}-a_{i_{m}}\leq 1/T$}
			{post $a_{i_{m}}$}
		\Else{
			post $X_{m}=a_{i_{m}}+n_{i_{m}}\e_{i_{m}}$ for $\big\lceil {8\ln(\delta^{-1})} / {\gamma^{2}} \big\rceil$ rounds and compute $\Dover(X_{m})$\label{state:post-stoc}\;
			\If(\tcp*[f]{up prices until surpassing the closest $v_{j}$}){$\Dover(a_{i_{m}})-\Dover(X_{m})<\gamma/2$\label{state:if_still_stoc}}{
				\lIf{$X_{m}+\e_{i_{m}}<b_{i_{m}}$}
					{update $n_{i_{m}}\gets n_{i_{m}}+1$}
				\lElse(\tcp*[f]{shrink the interval})
					{update $a_{i_{m}}\gets X_{m}$, $n_{i_{m}}\gets0$, $\e_{i_{m}}\gets\e_{i_{m}}^{2}$}
			}
			\Else({(denoting $a_0=\Dover(0)=0$)}){
				\If(\tcp*[f]{new valuation}){$\forall i \neq i_m$, $\mathrm{sign}(a_i-X_m) \big(\Dover(a_{i})-\Dover(X_{m})\big) \ge \gamma/2$\label{state:if_branch_stoc}}{
					$\kappa_{m}\gets\kappa_{m-1}+1$,
					$a_{\kappa_{m}}\gets X_{m}$, $b_{\kappa_{m}}\gets b_{i_{m}}$, $n_{\kappa_{m}}\gets1$, $\e_{\kappa_{m}}\gets\e_{i_{m}}$\;
				}
				update $a_{i_{m}}\gets X_{m}-\e_{i_{m}}$, $b_{i_{m}}\gets X_{m}$, $n_{i_{m}}\gets0$, $\e_{i_{m}}\gets\e_{i_{m}}^{2}$\label{state:shrink2-stoc}\tcp*{shrink interval}
			}
		}
	}
\caption{\label{alg:stoc_gamma}}
\end{algorithm2e}

Our algorithm receives as input the time horizon $T$, a lower bound $\gamma$ on $\pmin = \min_i p_i$, and a confidence parameter $\delta$. Given these parameters, the number of macrosteps is defined as the biggest $M_{\gamma}\in\N$ satisfying $T\geq M_{\gamma}\lceil 8\ln(\delta^{-1})/\gamma^{2}\rceil$. The fraction of accepted offers of price $x$ during the $m$-th macrostep (in which $x$ is offered) is denoted by $\Dover_{m}(x)$. Our algorithm (Algorithm~\ref{alg:stoc_gamma}) is very similar to Algorithm~\ref{alg:constant_val_K}, so we only highlight the main differences.

First, note that references to steps $t$ are replaced by references to macrosteps $m$; in particular, $\kappa_m$ is the number of allocated intervals after $m$ macrosteps. In line~\ref{state:stoc-greedy}, the selected arm $i_m$ is now the one maximizing, over intervals $[a_i,b_i]$, the product $b_iU_i$. The quantity $U_i$ is the upper confidence bound
\[
	U_i = \Dhat_m(i) + \frac{1}{b_i}\sqrt{\frac{\ln(\delta^{-1})}{N_m(i)}}
\] 
where $N_m(i)$ is $\big\lceil 8\gamma^{-2}\ln\delta^{-1} \big\rceil$ (if $i>1$, which takes into account the macrostep in which interval $i$ was allocated) plus the total number of times that $i$ was picked in the first $m-1$ macrosteps, ignoring the steps occurring in all macrosteps when line~\ref{state:shrink2-stoc} was executed. $\widehat{D}_m(i)$ is the fraction of accepted offers during these $N_m(i)$ steps. In line~\ref{state:if_still_stoc}, a new valuation is detected when the difference between demands is bigger than $\gamma/2$. Finally, in line~\ref{state:if_branch_stoc} a new interval is allocated if the newly discovered demand differs from all previously detected demands by at least $\gamma/2$.
\begin{theorem}
\label{th:analisys-kval}
If Algorithm~\ref{alg:stoc_gamma} is run on an unknown number $K$ of pairs $(v_1,p_1)\dots,(v_K,p_K)$ with input parameters $\gamma \le \min_k p_k$ and $\delta = T^{-2}$, then its regret satisfies
\[
	R_T
\le
	 \sum_{i \colon \Delta_i > 0} \frac{4\ln T}{\Delta_i}  + \co\left(\frac{K\ln T}{\gamma^2}\ln\ln T\right)~.
\]
\end{theorem}
\begin{proof}
Without loss of generality, assume $M_{\gamma}B_{\gamma}=T$ where $B_{\gamma} \ge 8\ln(\delta^{-1})/\gamma^{2}$ is the length of a macrostep. Hence, for any given price $0 \le x \le 1$, Hoeffding's inequality implies $\big\lvert \Dover(x)-D(x) \big\rvert \le \gamma/4$ with probability at least $1-2\delta$. Therefore, if $ D(x)-D(y) \geq \gamma$, then
$\Dover(x)-\Dover(y) \ge \gamma/2$ with probability at least $1-4\delta$. Moreover, if $D(x)=D(y)$, then $\Dover(x)-\Dover(y)  \le \gamma/2$ with probability at least $1-4\delta$. Since at each macrostep of the algorithm we perform at most $K+1$ comparisons between $\Dover(x)$ and $\Dover(y)$ for pairs of points $x,y$ (lines~\ref{state:if_still_stoc} and~\ref{state:if_branch_stoc}), the probability that, for at least one of these comparisons, we have
\begin{equation}
\label{eq:event}
	\left(\big\lvert\Dover_m(x)-\Dover_m(y)\big\rvert < \frac{\gamma}{2} \;\wedge\; \lvert D(x)-D(y) \rvert \ge \gamma \right) \;\text{or}\; \left( \lvert \Dover_m(x)-\Dover_m(y) \rvert > \frac{\gamma}{2} \;\wedge\; D(x)=D(y)\right)
\end{equation}
is at most $4(K+1)\delta$. Let $\scB$ the event that~(\ref{eq:event}) occurs for at least one comparison in at least one macrostep. Then $\Pr(\scB) \le 4(K+1)M_{\gamma}\delta$.

Assume $\scB$ does not occur. Recall that $v_{\mu(i)}$ is the smallest valuation in $[a_i,b_i]$. Since $p_{\mu(i)} \ge \gamma$ by hypothesis, event $X_{m} > v_{\mu(i)}$ implies that the test in line~\ref{state:if_still_stoc} is false, and therefore line~\ref{state:shrink2-stoc} is executed. Therefore, assuming event~(\ref{eq:event}) never occurs, the macrosteps of Algorithm~\ref{alg:stoc_gamma} with feedback $r_t(X_t)$ are equivalent to the steps of Algorithm~\ref{alg:constant_val_K} run with feedback $X_tD(X_t)$. In particular, Lemma~\ref{lem:equiv} applies to the macrosteps of Algorithm~\ref{alg:stoc_gamma}.

Let $n_{m}(i)$ be the number of macrosteps (in the first $m-1$ macrosteps) where $i$ was picked. Similarly, let $\os_{m}(i)$ be the number of macrosteps (in the first $m-1$ macrosteps) when $i$ was picked and $X_{m} > v_{\mu(i)}$. Then we have $N_{m}(i)=B_{\gamma}\big(n_{m}(i)-\os_{m}(i)\big)$. Now note that $\Dhat_m(i)$ is the sample mean of a Bernoulli of parameter $D(v_{\mu(i)})$ because it is computed over $N_{m}(i)$ points sampled between $a_i$ and $v_{\mu(i)}$. Fix a suboptimal valuation $v_k$ and a macrostep $m$ such that $\mu(i_m) = k$. Let $i^{\star}$ be such that $v^{\star}\in \big[a_{i^{\star}},b_{i^{\star}}\big]$. Then,
\begin{align*}
	i_m \neq i^{\star}
&\Longrightarrow
	b_{i^{\star}}U_{i^{\star}} \le b_{i_m} U_{i_m}
\\ &\Longleftrightarrow
	\left(b_{i^{\star}}\Dhat_m(i^{\star}) + \sqrt{\frac{\ln(\delta^{-1})}{N_m(i^{\star})}}\right)
\le
	\left(b_{i_m}\Dhat_m(i_m) + \sqrt{\frac{\ln(\delta^{-1})}{N_m(i_m)}}\right)
\\ &\Longrightarrow
	\left(v^{\star}\Dhat_m(i^{\star}) + \sqrt{\frac{\ln(\delta^{-1})}{N_m(i^{\star})}}\right)
\le
	\left(\left(v_k + \frac{2}{n_m(i_m)}\right)\Dhat_m(i_m) + \sqrt{\frac{\ln(\delta^{-1})}{N_m(i_m)}}\right)
\end{align*}
where in the last step we used Lemma~\ref{lem:cautious_shrinking_rate} in Appendix~\ref{s:shrinking}. Now recall that $n_m(i_m) \ge N_m(i_m)/B_{\gamma}$. Hence,
\[
	i_m \neq i^{\star}
\Longrightarrow
	\left(v^{\star}\Dhat_m(i^{\star}) + \sqrt{\frac{\ln(\delta^{-1})}{N_m(i^{\star})}}\right)
\le
	\left(v_k\Dhat_m(i_m) + \frac{2 B_{\gamma}}{N_m(i_m)} + \sqrt{\frac{\ln(\delta^{-1})}{N_m(i_m)}}\right)~.
\]
Observe that $\E\big[\Dhat_m(i^{\star})\big] = D\big(v_{\mu(i^{\star})}\big) \ge D(v^{\star})$ and $\E\big[\Dhat_m(i_m)\big] = D(v_k)$. Moreover, the two quantities $\sqrt{\big(\ln(\delta^{-1})\big)\big/\big(N_m(i^{\star})\big)}$ and ${2B_{\gamma}}/{N_m(i_m)} + \sqrt{\big(\ln(\delta^{-1})\big)\big/\big(N_m(i_m)\big)}$ play the role of upper confidence bounds for the estimates $v^{\star}\Dhat_m(i^{\star})$ and $v_k\Dhat_m(i_m)$. Therefore, we can apply a modification of the analysis of UCB1 \cite[Proof of Theorem~1]{auer2002finite} to $K$ arms with reward expectations $v_kD(v_k)$ for $k\in\{1,\dots,K\}$, and such that the upper confidence bound for any suboptimal arm $k$ is inflated by ${2B_{\gamma}}/{N_m(i_m)}$. (In fact Lemma~\ref{lem:ucb} in Appendix~\ref{s:ucb} is stronger than what we need, because $v^\star$ always belongs to some interval $[a_{j^\star},b_{j^\star}]$ but not all suboptimal valuations $v_k$ are the smallest valuation of the interval $[a_{j_k},b_{j_k}]$ they belong to.)  In particular, recalling that $B_{\gamma} = 8(\ln(\delta^{-1}))/\gamma^2$ and recalling also our assumption in $\scB$, we apply Lemma~\ref{lem:ucb} in Appendix~\ref{s:ucb} with $\alpha = 16$. This gives
\[
	B_{\gamma}\,\E\left[\Ind{\overline{\scB}}\sum_{m \colon \mu(i_m)=k} \Ind{i_m \neq i^{\star}} \right]
\le
	1 + \left( (\delta T)^2 + \frac{64}{\gamma^2}\right)2K\ln(\delta^{-1}) + \sum_{k\colon \Delta_k > 0} \frac{4\ln(\delta^{-1})}{\Delta_k}~.
\]
where $\Delta_k = v^{\star}D(i^{\star}) - v_kD(v_k) > 0$ and $\overline{\scB}$ is the complement of $\scB$.\footnote{
The factor $\Ind{\overline{\scB}}$ inside the expectation is needed to reduce the problem to an instance of a standard stochastic bandit. It can be conveniently dropped in the analysis of Lemma~\ref{lem:ucb}.
}
Because Lemma~\ref{lem:ucb} bounds the number of steps in which a suboptimal arm is selected, we multiplied by $B_{\gamma}$ the right-hand side of the above, thus converting macrosteps $m$ in steps $t$. The fact that we prevent the algorithm from switching arm within each macrostep is not an issue. Indeed, the proof of the Lemma works irrespective to whether the decision of pulling a different arm is made at every macrostep as opposed to every step. In particular, the proof establishes that after each suboptimal arm is selected order of $(\ln T)/\gamma^2$ times, corresponding to a constant number of macrosteps, the probability of pulling any suboptimal arm ever again becomes tiny, of order $T^{-2}$.

Similarly to the proof of Theorem~\ref{th:oracle}, introduce $\cm_k = \big\{ m \le M_{\gamma} \mid v_k\in[a_{i_m},b_{i_m}] \big\}$. As argued above, we may apply Lemma~\ref{lem:equiv} to the macrosteps of Algorithm~\ref{alg:stoc_gamma}. Hence, bound~(\ref{eq:equivbound}) applies with $\ct_k$ replaced by $\cm_k$. Therefore, with probability at least $1-4(K+1)M_{\gamma}\delta$, the regret over the $T$ steps (recall that we repeatedly post the same price in each step of a macrostep) is bounded by
\begin{align}
\nonumber
	B_{\gamma}\E&\left[\sum_{m=1}^{M_{\gamma}} \Bigl(v^{\star}D(v^{\star}) - X_m D(X_m)\Bigr)\right]
\\&\le
\nonumber
	B_{\gamma}\E\left[\sum_{k=1}^K \sum_{m \,:\, \mu(i_m)=k} \Bigl(v^{\star}D(v^{\star}) - v_k D(v_k)\Bigr) + \sum_{k=1}^K \sum_{m\in\cm_k} \Bigl(v_k D(v_k) - X_m D(X_m)\Bigr)\right]
\\&\le
\nonumber
	B_{\gamma}\sum_{k=1}^K \Delta_k \E\left[\Ind{\overline{\scB}} \sum_{m \,:\, \mu(i_m)=k} \Ind{i_m \neq i^{\star}}\right] + T\Pr(\scB) + B_{\gamma}\sum_{k=1}^K\big(3\ln\ln T_k + 8\big)
	\tag{using~(\ref{eq:equivbound})}
\\&\le
\label{eq:final-stoc-bound}
	1 + \left( (\delta T)^2 + \frac{64}{\gamma^2}\right)2K\ln(\delta^{-1}) + \sum_{k\colon \Delta_k > 0} \frac{4\ln(\delta^{-1})}{\Delta_k} + T\Pr(\scB) + B_{\gamma}K\big(3\ln\ln T + 8\big)~.
\end{align}
Finally, in order to bound $T\Pr(\scB) \le 4(K+1)TM_{\gamma}\delta = (K+1)(T\gamma)^2\delta/(2\ln\delta^{-1})$, it is sufficient to set $\delta = T^{-2}$.
\end{proof}

We conclude this section by discussing the case of at most two valuations. We design an algorithm with regret of order $\log(T)/\Delta + \log(T)\log\log(T)$, which is (up to the $\log \log$ term) as if the exact values of $v_1$ and $v_2$ were known in advance! This is achieved by leveraging some properties of the smallest and the biggest valuation. For example, any offer of a price lower or equal to $v_1$ is deterministically accepted and all offers above $v_2$ are always rejected. If on the other hand a price $x\in(v_1,v_2]$ is offered, the probability that that price is accepted is exactly $p_2$, which is enough to reconstruct the entire distribution $(p_1,p_2)$ on $\{v_1, v_2\}$. Furthermore, the suboptimality gap $\Delta$ is always equal to $\lvert v_1 - p_2v_2 \rvert$. 

Other than the result itself, we believe the techniques used in designing and analyzing the algorithm could be of interest on their own. Theorem~\ref{thm:mme} in particular gives a way to compute a high-probability multiplicative estimate of the unknown expectation $\mu >0$ of any $[0,1]$-valued random variable using only $\co\big(\frac{1}{\mu}\big)$ samples. We now state the result. All the details about the algorithm and its subroutines, their pseudocodes, and the remaining theoretical results are presented in Appendix~\ref{s:AppendixA}.
\begin{theorem}\label{TH:twoval-prob} If Algorithm \ref{algo:twoval-stoc} (see Appendix~\ref{s:AppendixA}) is run with input parameter $\delta = T^{-2}$ on an unknown instance $(v_1,p_1)$ and $(v_2,p_2)$, then its regret satisfies 
$
	R_T 
= 
	\co \big( {\log(T)}/{\Delta} + (\log T) (\log \log T) \big),
$
where the first term is zero when $\Delta = |p_2 v_2 - v_1|$ is zero.
\end{theorem}
\section{Conclusions}
\label{s:concl}
In this work we initiated an investigation of stochastic dynamic pricing in a setting in which the distribution of buyers' private values is supported on a finite set of points in $[0,1]$, where the number and location of these points is unknown to the seller. We studied the seller's regret in distribution-free and distribution-dependent settings, proving upper and lower bounds that show interesting connections to both the dynamic pricing setting of \citet{kleinberg2003value} and the standard stochastic $K$-armed bandit setting. We also proved some preliminary results for the nonstochastic version of our model when there are two valuations but only one is unknown (Appendix~\ref{s:twoval-adv}).

Our work leaves some interesting questions open. Can we prove a distribution-free upper bound of order $\sqrt{KT}$ that does not depend on the locations of buyers' valuations? Can we prove a distribution-dependent upper bound without any prior knowledge at all for $K$ larger than two? Can we obtain a $\sqrt{KT}$ regret bound in the nonstochastic setting when $K\ge2$ and all valuations are unknown?

\appendix


\section{Lower Bound}
\label{s:lb-deferred}

In this section we prove the lower bounds (Theorems~\ref{th:kl-adapted}~and~\ref{th:lower1}) stated in Section~\ref{s:lower}. \citet{kleinberg2003value} showed that $R_{T}=\Omega(T^{2/3})$ if $T\le K^{3}$ by building a distribution over a set of
$\e$-spaced valuations $v_{1},\l,v_{K}\in\big[\frac{1}{2},1\big]$.
A key technical property needed in their proof is that $\kl\big(\frac{1}{2v},\frac{9}{10}\frac{1}{2v}+\frac{1}{10}\frac{1}{2(v-\e)}\big)\le c\e^{2}$
for some constant $c$ independent of $\e$ and for all $v\ge3/4$.
We begin by showing that such construction only works if $K$ is large compared to $T$.
\begin{lemma}
For all $K\ge1$, for all $\e\in\left(0,\frac{1}{2K}\right]$, and
for all $k\in\{1,\l,K\}$, denoting $v=\frac{1}{2}+k\e$, 
\[
\kl\left(\frac{1}{2v}\ \Big\Vert\ \frac{9}{10}\frac{1}{2v}+\frac{1}{10}\frac{1}{2(v-\e)}\right)>\frac{\e}{800k}~.
\]
\end{lemma}
\begin{proof}
\allowdisplaybreaks
Fix any $K\ge1$, $\e\in\left(0,\frac{1}{2K}\right]$, and $k\in\{1,\l,K\}$.
Denoting $v=\frac{1}{2}+k\e$,
\begin{align*}
 & \kl\left(\frac{1}{2v}\ \Big\Vert\ \frac{9}{10}\frac{1}{2v}+\frac{1}{10}\frac{1}{2(v-\e)}\right)\\
 & \qquad=\frac{1}{2v}\ln\left(\frac{\frac{1}{2v}}{\frac{9}{10}\frac{1}{2v}+\frac{1}{10}\frac{1}{2(v-\e)}}\right)+\left(1-\frac{1}{2v}\right)\ln\left(\frac{1-\frac{1}{2v}}{1-\left[\frac{9}{10}\frac{1}{2v}+\frac{1}{10}\frac{1}{2(v-\e)}\right]}\right)\\
 & \qquad=\frac{1}{2v}\ln\left(\frac{1}{1+\frac{\e}{10(v-\e)}}\right)+\frac{1}{2v}(2v-1)\ln\left(\frac{2v-1}{2v-1-\frac{\e}{10(v-\e)}}\right)\\
 & \qquad\ge\frac{1}{2}\left[-\ln\left(1+\frac{\e}{5+10(k-1)\e}\right)-2k\e\ln\left(1-\frac{1}{10k+20k(k-1)\e}\right)\right]\tag{\text{using }\ensuremath{v=\frac{1}{2}+k\e\le1}}\\
 & \qquad=\frac{1}{2}\sum_{n=1}^{\iopp}\frac{2k\e\left(\frac{1}{10k+20k(k-1)\e}\right)^{n}+\left(-\frac{\e}{5+10(k-1)\e}\right)^{n}}{n}\\
 & \qquad=\frac{\e}{8k\big(5+10(k-1)\e\big)^{2}}+\frac{\e^{2}}{4\big(5+10(k-1)\e\big)^{2}}+\frac{1}{2}\sum_{n=3}^{\iopp}\frac{2k\e\left(\frac{1}{10k+20k(k-1)\e}\right)^{n}+\left(-\frac{\e}{5+10(k-1)\e}\right)^{n}}{n}\\
 & \qquad>\frac{\e}{800k}+\frac{\e^{2}}{400}+\e\sum_{n=3}^{\iopp}\frac{\frac{1}{k^{n-1}}\frac{1}{\big(10+20(k-1)\e\big)^{n}}+\frac{(-1)^{n}\e^{n-1}}{2}\frac{1}{\big(5+10(k-1)\e\big)^{n}}}{n}\\
 & \qquad=\frac{\e}{800k}+\frac{\e^{2}}{400}+\e\sum_{n=3}^{\iopp}\frac{\frac{1}{\big(10+20(k-1)\e\big)^{n}}\left(\frac{1}{k^{n-1}}+(-1)^{n}(2\e)^{n-1}\right)}{n}\\
 & \qquad>\frac{\e}{800k}+\frac{\e^{2}}{400}+\e\sum_{n=3}^{\iopp}\frac{\frac{1}{\big(10+20(k-1)\e\big)^{n}}\left(\frac{1}{k^{n-1}}-(2\e)^{n-1}\right)}{n}\\
 & \qquad\ge\frac{\e}{800k}+\frac{\e^{2}}{400}~.\tag{\text{using }\ensuremath{k\le K\text{ and }\e\le\frac{1}{2K}}}
\end{align*}
This concludes the proof.
\end{proof}
In order to prove Theorem~\ref{th:kl-adapted}, we need the following lemma.
\begin{lemma}
\label{lem:kl-small}For all $p,q\in(0,1)$
\[
\kl\left(p\,\Vert\,q\right)\le\frac{(p-q)^{2}}{q(1-q)}~.
\]
In particular, for all $x\in(0,1)$ and all $\alpha\in[0,1-x)$, 
\begin{equation}
\kl\big(x\;\Vert\;x+\alpha\big)\le\frac{\alpha^{2}}{(x+\alpha)(1-x-\alpha)}~.\label{eq:kl-small}
\end{equation}
\end{lemma}
\begin{proof}
Fix any $p,q\in(0,1)$. Using $\ln(x)\le x-1$ for all $x>0$, 
\begin{align*}
\kl(p\,\Vert\,q) & =p\ln\left(\frac{p}{q}\right)+(1-p)\ln\left(\frac{1-p}{1-q}\right)\le p\frac{p-q}{q}-(1-p)\frac{p-q}{1-q}\\
 & =(p-q)\left(\frac{p}{q}-\frac{1-p}{1-q}\right)=\frac{(p-q)^{2}}{q(1-q)}~.
\end{align*}
\end{proof}
We now restate and prove Theorem~\ref{th:kl-adapted}.
\begin{theorem}
\label{th:kl-adapted-appe} For any number of valuations $K\ge3$
and all time horizons $T\ge K^{3}$ there exist $K$ pairs $\big(v_{1},p(v_{1})\big),\l,\big(v_{K},p(v_{K})\big)$
such that the expected regret of any pricing strategy satisfies 
\[
R_{T}\ge\frac{1}{375}\sqrt{KT}~.
\]
\end{theorem}
\begin{proof}
For notational convenience, fix $K\ge2$ and define the set $\{v_{0},\l,v_{K}\}$ of $K+1$
valuations by
\[
v_{i}=\frac{1}{2}+\frac{i}{2K},\qquad\forall i\in\{0,\l,K\}~.
\]
Define the distribution $p_{0}$ on $\{v_{0},\l,v_{K}\}$ of the random variable $V_0$ by
\[
\P(V_{0}\geq v)=\sum_{i\colon v_{i}\ge v}p_{0}(v_{i})=\frac{1}{2v},\qquad\forall v\in\{v_{0},\l,v_{K}\}~.
\]
With this choice of demand curve, $v\P(V_{0}\ge v)=1/2$, i.e., each
valuation $v$ has the same expected revenue. Furthermore, the distribution
$v\mapsto p_{0}(v)$ satisfies the following: $p_{0}(v_{0})=\frac{1}{K+1}$;
$p_{0}$ decreases monotonically on $\left\{ v_{0},\l,v_{K-1}\right\} $,
$p_{0}(v_{K-1})=\frac{1}{2K-1}$, and $p_{0}(v_{K})=1/2$. Therefore
\begin{equation}
\frac{1}{2K}\le p_{0}(v)\le\frac{1}{K},\qquad\forall v\in\{v_{0},\l,v_{K-1}\}~.\label{eq:boud-p0}
\end{equation}
Now, for each $j\in\big\{\lceil K/2\rceil,\l,K\big\}$, define the
distribution $p_{j}$ by slightly lowering the probability of $v_{j-1}$
and upping the probability of $v_{j}$ by the same amount: 
\begin{equation}
p_{j}\left(v_{i}\right)=\begin{cases}
p_{0}(v_{i}), & i\in\{0,\l,K\}\m\{j-1,\,j\},\\
(1-4K\e)p_{0}(v_{j-1}), & i=j-1,\\
p_{0}(v_{j})+4K\e p_{0}(v_{j-1}), & i=j,
\end{cases}\label{eq:def-pj}
\end{equation}
where $\e\in\big(0,\frac{1}{40}\big)$ is a small constant determined
below. Note that if the buyers' valuations were distributed as $p_{j}$,
all valuations $v\neq v_{j}$ would have expected revenue $\frac{1}{2}$,
but $v_{j}$ whould have expected revenue at least $\frac{1}{2}+\e$
because of (\ref{eq:boud-p0}) and (\ref{eq:def-pj}). In order to
define the distribution of buyers' valuations $V = \big(V_{1},\l,V_{T}\big)$, let
$J$ be uniformly distributed over $\big\{\lceil K/2\rceil,\l,K\big\}$
(that is, the set of indices $i\in\{1,\l,K\}$ such that $v_{i}\ge\frac{3}{4}$).
The value of $J$ will give the ``good valuation'', that is the
valuation with the highest expected revenue. For all $t$, the distribution
of $V_{t}$ is determined by
\[
\P\big(V_{t}=v_{i}\mid J=j\big)=p_{j}(v_{i}),\qquad\forall i\in\{0,\l,K\},\forall j\in\big\{\lceil K/2\rceil,\l,K\big\}~.
\]
Denoting the seller's randomized strategy by $X=(X_{1},\l,X_{T})$
and applying Fubini's theorem, we obtain 
\[
R_{T}=\max_{k\in\{0,\l,K\}}\E_{X}\E_{J,V}\left[\sum_{t=1}^{T}r_{t}(v_{k})-\sum_{t=1}^{T}r_{t}(X_{t})\right]~.
\]
According to the previous identity, we can (an will!) lower bound
the internal expectation assuming that the seller's strategy is deterministic.
Furthermore, assume that the seller's pricing strategy only offers
prices in $\{v_{\lceil K/2\rceil},\l,v_{K}\}$ \textemdash since it
is counterproductive to offer a price outside of it as all other valuations
($v_{1},\l,v_{\lceil K/2\rceil-1}$ in particular) have smaller expected
revenues. Now let $N_{i}$ be the number of times the seller offer
valuation $v_{i}$, 
\[
N_{i}=\sum_{t=1}^{T}\I\{X_{t}=v_{i}\}~.
\]
By construction, each time the seller picks the ``good valuation'',
no regret is accrued; all other times at least $\e$ is lost. Therefore
\begin{equation}
\E_{J,V}\left[\sum_{t=1}^{T}r_{t}(v_{k})-\sum_{t=1}^{T}r_{t}(X_{t})\right]\ge\e\big(T-\E_{J,V}[N_{J}]\big)~.\label{eq:lower-bound-beginning}
\end{equation}
Denote by $Y_{t}$ the Bernoulli random variable $\I\{V{}_{t}\ge X_{t}\}$
which is $1$ if and only if the $t$-th buyer accepted the price
offered, $Y^{t}=(Y_{1},\l,Y_{t})$, and $Y=Y^{T}$. Denote by $q_{0}$
the distribution of $Y$ if buyer's valuations were distributed as
$p_{0}$ and by $q_{i}$ the distribution of $Y$ if buyer's valuations
were distributed as $p_{i}$. For any deterministic function $f\colon\{0,1\}^{T}\to[0,M]$,
\begin{align*}
\E_{V}\big[f(Y)\mid J=i\big]-\E_{0}[f(Y)] & =\sum_{b^{T}\in\{0,1\}^{T}}f(b^{T})\big(q_{i}(b^{T})-q_{0}(b^{T})\big)\\
 & \le\sum_{\substack{b^{T}\in\{0,1\}^{T}\\
q_{i}(b^{T})>q_{0}(b^{T})
}
}f(b^{T})\big(q_{i}(b^{T})-q_{0}(b^{T})\big)\\
 & \le M\sum_{\substack{b^{T}\in\{0,1\}^{T}\\
q_{i}(b^{T})>q_{0}(b^{T})
}
}\big(q_{i}(b^{T})-q_{0}(b^{T})\big)\\
 & \le M\sqrt{\frac{1}{2}\kl(q_{0}\;\Vert\;q_{i})}
\end{align*}
where $\E_{0}$ is the expectation with respect to distribution $p_{0}$
and in the last step we used Pinsker's inequality. Let 
$
q_{i}(b_{t}\mid b^{t-1})=p_{i}\left(Y_{t}=b_{t}\mid Y_{1}=b_{1},\l,Y_{t-1}=b_{t-1}\right)
$
and let $q_{0}(b_{t}\mid b^{t-1})$ be defined similarly. By the chain
rule of the relative entropy 
\begin{align*}
\kl(q_{0}\;\Vert\;q_{i}) & =\sum_{t=1}^{T}q_{0}(b^{t-1})\sum_{b^{t-1}\in\{0,1\}^{t-1}}\kl\big(q_{0}(b_{t}\mid b^{t-1})\;\Vert\;q_{i}(b_{t}\mid b^{t-1})\big)\\
 & =\sum_{t=1}^{T}q_{0}(b^{t-1})\sum_{b^{t-1}\colon X_{t}(b^{t-1})\neq v_{i}}\underbrace{\kl\big(q_{0}(b_{t}\mid b^{t-1})\;\Vert\;q_{i}(b_{t}\mid b^{t-1})\big)}_{=0}\\
 & \phantom{=}+\sum_{t=1}^{T}q_{0}(b^{t-1})\sum_{b^{t-1}\colon X_{t}(b^{t-1})=v_{i}}\kl\big(q_{0}(b_{t}\mid b^{t-1})\;\Vert\;q_{i}(b_{t}\mid b^{t-1})\big)
\end{align*}
where the relative entropy is zero when $X_{t}\neq v_{i}$ because
in that case $p_{i}(Y_{t}=1)=p_{0}(Y_{t}=1)$. If on the other hand,
$X_{t}=v_{i}$, for all $v_{i}\ge\frac{3}{4}$, 
\begin{align*}
\kl\big(q_{0}(b_{t}\mid b^{t-1})\;\Vert\;q_{i}(b_{t}\mid b^{t-1})\big) & =\kl\left(\frac{1}{2v_{i}}\ \Big\Vert\ \frac{1}{2v_{i}}+4K\e p_{0}(v_{j-1})\right)\le108\e^{2}
\end{align*}
where the last inequality follows by (\ref{eq:boud-p0}) and 
$
	\kl\big(x\;\Vert\;x+\alpha\big)\le {\alpha^{2}}{(x+\alpha)^{-1}(1-x-\alpha)^{-1}},
$
with $x=\frac{1}{2v_{i}}\in\big[\frac{1}{2},\frac{2}{3}\big]$ and
$\alpha=4K\e p_{0}(v_{j-1})\in[2\e,4\e]$. Therefore 
\[
\kl(q_{0}\;\Vert\;q_{i})\le108\e^{2}\sum_{t=1}^{T}q_{0}(b^{t-1})\sum_{b^{t-1}\colon X_{t}(b^{t-1})=v_{i}}1=108\e^{2}\sum_{t=1}^{T}p_{0}(X_{t}=v_{i})=108\e^{2}\E_{0}[N_{i}]~,
\]
where again, $\E_{0}$ is the expectation with respect to distribution
$p_{0}$. This gives 
\[
\E_{V}[f(Y)\mid J=i]\le\E_{0}[f(Y)]+\e M\sqrt{54\E_{0}[N_{i}]}~.
\]
Then, being for any deterministic online pricing strategy the random
variable $N_{i}$ a deterministic function of $Y$, 
$
\E_{V}[N_{i}\mid J=i]\leq\E_{0}[N_{i}]+\e T\sqrt{54\E_{0}[N_{i}]}
$.
Thus, using Jensen inequality, 
$
\E_{J,V}[N_{i}]\leq\E_{J}\E_{0}[N_{J}]+\e T\sqrt{54\E_{J}\E_{0}[N_{J}]}
$.
Using again Jensen inequality, Fubini's Theorem, and inequality (\ref{eq:lower-bound-beginning}),
\[
\E_{J,V}\E_{X}\left[\sum_{t=1}^{T}r_{t}(v_{k})-\sum_{t=1}^{T}r_{t}(X_{t})\right]\ge\e\left(T-\E_{J}\E_{0}\E_{X}[N_{J}]-\e T\sqrt{54\E_{J}\E_{0}\E_{X}[N_{J}]}\right)~.
\]
Since $\sum_{i=\lceil K/2\rceil}^{K}N_{i}=T$, we also have $\sum_{i=\lceil K/2\rceil}^{K}\E_{0}\E_{X}[N_{i}]=T$.
Using the fact that $K-\lceil K/2\rceil+1\ge\max\{3/2,K/2\}$, this implies 
\[
\E_{J}\E_{0}\E_{X}[N_{J}]=\frac{1}{K-\lceil K/2\rceil+1}\sum_{i=\lceil K/2\rceil}^{K}\E_{0}\E_{X}[N_{i}]\le\min\left\{ \frac{2}{3},\frac{2}{K}\right\} T~.
\]
Putting everything together, we get 
\[
R_{T}\ge\e\left(T-\frac{2}{3}T-\e T\sqrt{\frac{108T}{K}}\right)=\e T\left(\frac{1}{3}-\e\sqrt{\frac{108T}{K}}\right)~,
\]
which picking $\e=\frac{1}{6\sqrt{108}}\sqrt{K/T}$ so that $\e\sqrt{108T/K}=1/6$,
gives 
\[
R_{T}\ge\frac{1}{375}\sqrt{KT}
\]
as desired.
\end{proof}
In summation, Even if the technique used by \citet{kleinberg2003value} fails in our setting, it is  still possible to prove an analogous lower bound by changing some key aspects of their analysis, which in turn is based on the lower bound analysis of~\citep{auer2002nonstochastic}. First, valuations need to be distanced as much as possible ---this is the exact opposite of their construction, where valuations were placed $\e$-close to each others. Second, the base distribution is only perturbed by an appropriate small constant. Third, the ``good valuation'' is drawn from a sensible proper subset of valuations. 
\section{Cautious search}
\label{s:shrinking}
\citet{kleinberg2003value} were first to introduce a ``cautious search'' as an optimal algorithm for posted price with a single unknown evaluation. Similarly, our cautious search (Algorithm~\ref{alg:cautious_search}) proceeds in phases $s\in\{1,2,\ldots\}$ in which an interval $[a_{s},\,b_{s}]$ (initialized to $[0,1]$) and a step size $\e_{s}$ (initialized to $1/2$) are maintained. In a given phase $s$ of the algorithm, prices $a_{s}+\e_{s},\,a_{s}+2\e_{s},\,a_{s}+3\e_{s},\,\l$ are posted until one of them, say $X_{s}$, becomes bigger than
the hidden evaluation (overshooting). At this point a new phase begins: the interval becomes
$[a_{s+1},\,b_{s+1}]=[X_{s}-\e_{s},\,X_{s}]$, and the new step size becomes $\e_{s+1}=\e_{s}^{2}$. This process continues until the length of the interval is less than $1/T$. Then the left endpoint of the interval is picked for all remaining rounds.
\begin{algorithm2e}[t]
\caption{Cautious search\label{alg:cautious_search}}
	\LinesNumbered
	\SetKwInput{kwInit}{Initialization}
	\KwIn{Time horizon $T\in\mathbb{N}$.}
	\kwInit{set $a \gets 0$, $b \gets 1$, $n \gets 1$, $\e \gets 1/2$.}

	\For{$t\in\{1,\ldots T\}$}{
		post $X_{t}=a +n\e$ and get feedback $Z_t = \Ind{X_t \le v}$;

			\uIf(\tcp*[f]{undershooting}){$Z_t = 1$}{
				\lIf{$X_{t}+\e < b$}{update $n \gets n+1$}
				\lElse(\tcp*[f]{shrink the interval}){update $a \gets X_{t}$, $n \gets 1$, $\e\gets\e^2$}
			}
			\uElseIf(\tcp*[f]{overshooting}){$Z_t=0$}{ 
				update $a \gets X_{t}-\e$, $b \gets X_{t}$, $n \gets 1$, $\e\gets\e^{2}$\tcp*{shrink the interval}
			}
	}
\end{algorithm2e}
We now state two lemmas about the behavior of cautious search. The first one is proven in~\cite[Theorem 2.1]{kleinberg2003value}.
\begin{lemma}
\label{lem:KL}
The regret of Algorithm~\ref{alg:cautious_search} satisfies $\E\left[\sum_{t=1}^T r_t(v) - \sum_{t=1}^T r_t(X_t)\right] \le 3\ln\ln(T) + 8$.
Moreover, the number of overshootings is upper bounded by $\log\log T$.
\end{lemma}
The second lemma bounds the size of the interval as a function of the number of steps.
\begin{lemma}
\label{lem:cautious_shrinking_rate}
For all $m$, the size of an interval $[a_{s},b_{s}]$ after $m$ steps of Algorithm~\ref{alg:cautious_search} satisfies
\[
	b_{s}-a_{s}\leq\frac{2}{m}~.
\]
\end{lemma}
\begin{proof}
The worst case happens when the sequence $(b_{1}-a_{1},b_{2}-a_{2},\ldots)$ of interval endpoints takes values 
\begin{equation}
\left(1,1,\frac{1}{2},\frac{1}{2},\frac{1}{4},\frac{1}{4},\frac{1}{4},\frac{1}{4},\ldots,\frac{1}{2^{2^{n}}},\ldots,\frac{1}{2^{2^{n}}},\ldots\right)
\label{eq:worst_case_shrink_rate}
\end{equation}
where the general term $1/2^{2^{n}}$ is repeated $2^{2^{n}}$ times.
It is then sufficient to show that the inequality holds for all values
before a switch. Formally, that for all $n\in\{0,1,2,\ldots\}$
\[
	\frac{1}{2^{2^{n}}}\leq\frac{2}{2+\sum_{j=0}^{n}2^{2^{j}}}
\qquad\text{or, equivalently,}\qquad
2+\sum_{j=0}^{n}2^{2^{j}}\leq2\cdot2^{2^{n}}~.
\]
We prove this by induction on $n$. The case $n=0$ is trivial. If
the inequality holds for $n\in\{0,1,\ldots\}$, then
\[
2+\sum_{j=0}^{n+1}2^{2^{j}}=2+\sum_{j=0}^{n}2^{2^{j}}+2^{2^{n+1}}\leq2\cdot2^{2^{n}}+2^{2^{n+1}}=2^{2^{n}}\left(2+2^{2^{n}}\right)\leq2\cdot2^{2^{n+1}}~.
\]
This concludes the proof.
\end{proof}
The previous bound is unimprovable. Indeed in scenario~(\ref{eq:worst_case_shrink_rate}), for all $n\in\{0,1,\ldots\}$ 
\[
2^{2^{n}}<2+\sum_{j=0}^{n}2^{2^{j}}\leq2\cdot2^{2^{n}}
\]
and the second inequality is actually an equality for $n=0$.

\newcommand{\Xhat}{\widehat{X}}
\section{UCB with inflated confidence bounds}
\label{s:ucb}
In this section we prove a regret bound for UCB1 run with an oracle that systematically inflates the upper confidence bounds for suboptimal arms.
\begin{lemma}
\label{lem:ucb}
Consider a stochastic bandit problem with $K$ arms, i.i.d.\ rewards $X_t(k) \in [0,1]$ from each arm $k$, and average rewards $\mu_1,\dots,\mu_K$. Let $\Delta_k = \mu^{\star} - \mu_k$ where $\mu^{\star} = \mu_{i^{\star}}$ and $i^{\star}$ is the index of an optimal arm. Consider a UCB policy that at round $t$ selects arm $I_t$ defined by 
\[
	I_t = \argmax_{k\in\{1,\ldots,K\}} \Big( \Xhat_t(k) + c\big(N_t(k),k\big) \Big)
\]
(ties broken arbitrarily), where $\Xhat_t$ is the sample average of the rewards obtained from arm $k$ over the $N_t(k)$ times when the arm was chosen in rounds $1,\dots,t-1$ (initially, $N_1(k)=0$ for all arms) and
\[
	c(s,k)
= 
	\left\{ \begin{array}{cl}
		{\displaystyle \frac{\alpha\ln(\delta^{-1})}{\gamma^2 s} + \sqrt{\frac{\ln(\delta^{-1})}{s}} } & \text{if $k$ is suboptimal,}
	\\[2mm]
		{\displaystyle \sqrt{\frac{\ln(\delta^{-1})}{s}} } & \text{otherwise,}
	\end{array} \right.
\]
with $\alpha \ge 0$ and $c(s,k)=+\infty$ if $s=0$. Then
\[
	R_T
\le
	1+ \left( 2(\delta T)^2 + \frac{8\alpha\ln(\delta^{-1})}{\gamma^2}\right)K + \sum_{k\colon \Delta_k > 0} \frac{4\ln(\delta^{-1})}{\Delta_k}~.
\]
\end{lemma}
\begin{proof}
Pick any suboptimal arm $k$ 
 and $t \ge 2$. Note that $I_t = k$ implies
\[	
	\Xhat_t(i^\star) + c\big(N_t(i^\star),i^\star\big) \le \Xhat_t(k) + c\big(N_t(k),k\big)
\]
which in turn imply
\begin{align*}
	\Bigl(\Xhat_t(i^\star) \le \mu^* - c\big(N_t(i^\star),i^\star\big) \Bigr)
\,\vee\,
	\Bigl(\Xhat_t(k) \ge \mu_k + c\big(N_t(k),k\big) \Bigr)
\,\vee\,
	\Bigl(c\big(N_t(k),k\big) > \Delta_k/2\Bigr)~.
\end{align*}
Using standard Chernoff bounds, we can write
\begin{align*}
	\sum_{t=2}^T\Pr\Bigl(\Xhat_t(i^\star) \le \mu^* - c\big(N_t(i^\star),i^\star\big)\Bigr)
&\le
	\sum_{t=2}^T\Pr\Bigl( \exists s \in \{1,\ldots,t-1\}, \ \Xhat_t(i^\star) \le \mu^* - c(s,i^\star) \Bigr)
\\&\le
	\sum_{t=2}^T\sum_{s=1}^{t-1} \exp\left(-2s\frac{\ln(\delta^{-1})}{s}\right) \le T^2\delta^2
\end{align*}
and
\begin{align*}
	\sum_{t=2}^T\Pr\Bigl(\Xhat_t(k) \ge \mu_k + c\big(N_t(k),k\big)\Bigr)
&\le
	\sum_{t=2}^T\Pr\Bigl( \exists s \in \{1,\ldots,t-1\}, \ \Xhat_t(k) \ge \mu_k + c(s,k) \Bigr)
\\&\le
	\sum_{t=2}^T\Pr\left( \exists s \in \{1,\ldots,t-1\}, \ \Xhat_t(k) \ge \mu_k + \sqrt{\frac{2\ln(\delta^{-1})}{s}} \right)
\\&\le
	\sum_{t=2}^T\sum_{s=1}^{t-1} \exp\left(-2s\frac{\ln(\delta^{-1})}{s}\right) \le T^2\delta^2~.
\end{align*}
It remains to control $\Ind{c\big(N_t(k),k\big) > \Delta_k/2}$ when $I_t = k$. We now show that 
\[
	\sum_{t=2}^T \Ind{c\big(N_t(k),k\big) > \Delta_k/2} \le 4\left( \frac{2\alpha}{\gamma^2\Delta_k} + \frac{1}{\Delta_k^2}\right)\ln(\delta^{-1})~.
\]
If $k$ is chosen 
$s>0$ times in the first $t-1$ steps, then $N_t(k) = s$. Thus $c(s,k) > \Delta_k/2$ implies
\begin{equation}
\label{eq:ucb-first}
	\frac{\alpha\ln(\delta^{-1})}{\gamma^2s} + \sqrt{\frac{\ln(\delta^{-1})}{s}}
>
	\frac{\Delta_k}{2}~.
\end{equation}
We now prove that $s$ must be smaller than
\[
	4\left( \frac{2\alpha}{\gamma^2\Delta_k} + \frac{1}{\Delta_k^2}\right)\ln(\delta^{-1})
\]
for this to happen. If $\alpha=0$ this is trivially true. To see that this still true for $\alpha>0$, note that with this assumption~(\ref{eq:ucb-first}) is equivalent to
\[
	\sqrt{\frac{\ln(\delta^{-1})}{s}}
>
	\frac{-1 + \sqrt{1+2\Delta_k\alpha/\gamma^2}}{2\alpha/\gamma^2}~.
\]
Set $x = 2\Delta_k\alpha/\gamma^2 > 0$ so that the above can be rewritten as
\[
	\sqrt{\frac{\ln(\delta^{-1})}{s}}
>
	\frac{\Delta_k\left(\sqrt{1+x}-1\right)}{x}
\]
or, squaring both sides,
\[
	\frac{s}{\ln\delta^{-1}}
<
	\frac{x^2}{\Delta_k^2\left(\sqrt{1+x}-1\right)^2}~.
\]
We now prove that
\[
	\frac{x^2}{\Delta_k^2\left(\sqrt{1+x}-1\right)^2}
\le
	\frac{4}{\Delta_k^2}(x+1)~.
\]
Indeed, the above is equivalent to
\[
	\frac{x}{\sqrt{1+x}-1} \le 2 \sqrt{1+x}
\]
which holds because
\[
	\frac{x}{\sqrt{1+x}-1}  = \frac{\left(\sqrt{1+x}+1\right)\left(\sqrt{1+x}-1\right)}{\sqrt{1+x}-1} \le 2 \sqrt{1+x}~.
\]
Setting $\delta = T$, the regret is therefore bounded as follows
\[
	R_T
\le
	1 + \sum_{k\colon \Delta_k > 0} \Delta_k \sum_{t=2}^T \Pr(I_t=k)
\le
	1 + 2KT^2\delta^2 + \frac{8\alpha K}{\gamma^2}\ln(\delta^{-1}) + \sum_{k\colon \Delta_k > 0} \frac{4\ln(\delta^{-1})}{\Delta_k}~.
\]
This concludes the proof.
\end{proof}


\newcommand{\Dbar}{\Dover}

\section{Two valuations}
\label{s:AppendixA}
In this section we present all key results related to subroutines of Algorithm~\ref{algo:twoval-stoc} and give a formal proof of Theorem~\ref{TH:twoval-prob}.
\subsection*{Noisy Cautious Search}
This procedure is a variant of the cautious search described in Appendix~\ref{s:shrinking}. It identifies the location of a valuation $v_i$ with high probability and low regret whenever a lower bound $\gamma_i$ on its probability $p_i$ is known in advance. 
During the search, each price is posted for $\big\lceil{\ln(\delta)}/{\ln(1-\gamma_i)}\big\rceil$ times in a row, where $\delta$ is a confidence parameter. We call such a sequence of consecutive rounds a \emph{macrostep}. 
For $i=1$, we say that a macrostep is a \emph{failure} if at least one price is rejected, it is a \emph{success} if all prices are accepted, and the algorithm makes a \emph{mistake} if the macrostep is a success but the price offered is strictly bigger than $v_1$.
For $i=2$, we say that a macrostep is a \emph{failure} if no price is accepted, it is a \emph{success} if at least one price is accepted, and the algorithm makes a \emph{mistake} if the macrostep is a failure but the price offered is at most $v_2$.

The Noisy Cautious Search for a valuation $v_i$ proceeds in phases and begins by offering $1/2$ during the first macrostep. During each phase $n\ge 0$, if the last macrostep was a success, the price offered is increased by $2^{-2^n}$. As soon as a macrostep is a failure, phase $n$ ends and phase $n+1$ begins by offering the price of the last successful macrostep, plus $2^{-2^{n+1}}$. After $ \lceil \log_2\log_2 T \rceil$ phases, the price of the last successful macrostep is offered for all remaining rounds.
\begin{algorithm2e}[t]
	\caption{Noisy Cautious Search}
	\LinesNumbered
	\SetKwInput{kwInit}{Initialization}
	\KwIn{confidence parameter $\delta\in(0,1)$, valuation index $i\in\{1,2\}$, lower bound $\gamma_i\in(0,1)$.}
	\kwInit{set $a\gets 0$, $b\gets 1$.}
	\For(\tcp*[f]{phases}){$s \in \big\{ 0, 1, \l, \lceil \log_2\log_2 T\rceil \big\}$}{
		set $n\gets 1$, $\e_s \gets 2^{-2^s}$, $\Dbar \gets 1$;

		\While{$(a+ n \e_s < b) \wedge \big[ ( i = 1 \wedge \Dbar = 1 ) \vee ( i = 2 \wedge \Dbar > 0 ) \big]$}{
			offer price $a + n \e_s$ for $\big\lceil{\ln(\delta)}/{\ln(1-\gamma_i)}\big\rceil$ rounds\tcp*{a macrostep}
			update $n \gets n+1$ and the sample mean $\Dbar$ of $D\big(a + (n-1) \e_s\big)$;
		}
		update $a \gets a+ (n-1) \e_s$,  $b \gets a+ n \e_s$;
	}
	offer $a$ for all remaining rounds;
\end{algorithm2e}
\begin{lemma} 
\label{lm:noisy-cs}
The Noisy Cautious Search for $v_i$ with parameters $i,\delta, \gamma_i$ satisfies the following:
\begin{enumerate}[topsep=0pt,parsep=0pt,itemsep=0pt]
\item \label{c:one} the price offered during each macrostep $m$ is $2/m$-close to $v_i$ with probability at least $1 - m\delta$;
\item \label{c:two} the total reward accumulated by the end of macrostep $m$ is at least
\[
	\big(m v_i D(v_i) - 3 (\ln \ln T) -8\big)\frac{\ln \delta}{\ln(1-\gamma_i)}
\]
with probability at least $1 - m \delta$.
\end{enumerate}
\end{lemma}
\begin{proof}
Claim~\ref{c:one} follows by Lemma~\ref{lem:cautious_shrinking_rate} and the fact that the probability of making a mistake during each macrostep is at most $\delta$ by Chernoff inequality for Bernoulli random variables. Similarly, claim~\ref{c:two} follows by Lemma~\ref{lem:KL} and, again, Chernoff inequality.
\end{proof}
\subsection*{Capped Mean Estimation} 
We begin this section by providing a method to find a high-confidence multiplicative estimate of the expectation $\mu$ of any $[0,1]$-valued random variable, using only $\co\big(\ln(1/\delta)/\mu\big)$ samples. Most notably, the expectation $\mu$ need \emph{not} be known in advance. With our novel technique, we improve upon \citet[Lemma 13]{berthet2017fast}, that proved a similar risult using $\co\big(\ln(1/\delta)/\mu^2\big)$ samples. This result will be pivotal for our analysis and we believe it will also be valuable in its own right. 
For any set $X_{1},\dots,X_{T}$ of random variables, we denote by
\[
	\overline{X}_{t}
=
	\frac{1}{t}\sum_{s=1}^{t}X_{s}
\qquad\text{and}
	\qquad S_{t}^{2}
=
	\frac{1}{t-1}\sum_{s=1}^{t}\big(X_{s}-\overline{X}_{t}\big)^{2}
\]
the \emph{sample mean} and the \emph{sample variance} of the first
$t$ random variables. The following result is a straightforward consequence
of the empirical Bernstein bound and the confidence bound for standard
deviation proven in \cite[Theorems~4,~10]{MaurerPontil09}.
\begin{theorem}
\label{thm:MauerPontil}Let $X_{1},\l,X_{T}$ be a set of $[0,1]$-valued
i.i.d.\ random variables with expectation $\mu$ and standard
deviation $\sigma$. For all $\delta\in(0,1)$ and all $t\in\{2,\dots,T\}$,
the two following conditions hold simultaneously with probability
at least $1-3\delta$ 
\[
	\big|\overline{X}_{t}-\mu\big|
\leq
	\sqrt{2}S_{t}\left(\frac{\ln(1/\delta)}{t}\right)^{1/2}+\frac{7}{3}\frac{\ln(1/\delta)}{t-1}
\qquad\text{and}\qquad 
	S_{t}
\le
	\sigma+\sqrt{2}\left(\frac{\ln(1/\delta)}{t-1}\right)^{1/2}~.
\]
\end{theorem}
We can now prove our multiplicative mean estimation theorem.
\begin{theorem}[Multiplicative mean estimation]
\label{thm:mme}
 Let $X_{1},\ldots,X_{T}$ be a set of $[0,1]$-valued i.i.d.\ random
variables with expectation $\mu>0$ and standard deviation $\sigma$.
For all $\delta\in(0,1)$ and all $\alpha\geq0$, if $T\geq t_{0},$
where 
\[
t_{0}=\left\lceil \frac{\alpha+2}{3\mu}\ln\left(\frac{1}{\delta}\right)\left(\sqrt{9\alpha^{2}+114\alpha+192}+3\alpha+19\right)\right\rceil +2=\mathcal{O}\left(\frac{\alpha^{2}}{\mu}\ln\frac{1}{\delta}\right)
\]
 and $\tau=\tau(T,\delta,\alpha)$ is the smallest time $t\in\{2,\ldots,T\}$
such that 
\begin{equation}
\frac{\overline{X}_{t}}{\alpha+1}\geq\sqrt{2}S_{t}\left(\frac{\ln(1/\delta)}{t}\right)^{1/2}+\frac{7}{3}\frac{\ln(1/\delta)}{t-1}\label{eq:bernoulli_condition-1}
\end{equation}
 then, with probability at least $1-3(T-1)\delta$,
\begin{enumerate}[topsep=0pt,parsep=0pt,itemsep=0pt]
\item $\tau\leq t_{0}$,
\item for all $t\in\{2,\ldots,T\}$ such that (\ref{eq:bernoulli_condition-1})
holds, 
\end{enumerate}
\begin{equation}
\label{eq:mme}
	\left(\frac{\alpha}{\alpha+1}\right)\overline{X}_{t}
<
	\mu
<
	\left(\frac{\alpha+2}{\alpha+1}\right)\overline{X}_{t}~.
\end{equation}
\end{theorem}
\begin{proof}
Denote for all $t\in\{2,\ldots,T\}$, $c_{t}=\sqrt{2\,S_{t}^{2}\ln(1/\delta)/t}+(7/3)\ln(1/\delta)/(t-1).$
By Theorem~\ref{thm:MauerPontil}, the \emph{good} event
\[
G=\left\{ \forall t\in\{2,\ldots,T\},\quad\overline{X}_{t}-c_{t}<\mu<\overline{X}_{t}+c_{t}\quad\text{and}\quad S_{t}\le\sigma+\sqrt{2\ln(1/\delta)/(t-1)}\right\} 
\]
 has probability $\P(G)\geq1-3(T-1)\delta.$ For all outcomes in $G$
and all $t\in\{2,\ldots,T\},$ 
\[
\overline{X}_{t}<(\alpha+1)c_{t}\iff\mu-c_{t}<\overline{X}_{t}<(\alpha+1)c_{t}\implies\mu<(\alpha+2)c_{t}\implies t<t_{0}
\]
hence $\tau\leq t_{0}$. This implies that for all outcomes in $G$
and all $t\in\{1,\ldots,T\}$ such that $\overline{X}_{t}\geq(\alpha+1)c_{t},$
\[
\left(\frac{\alpha}{\alpha+1}\right)\overline{X}_{t}=\overline{X}_{t}-\frac{\overline{X}_{t}}{\alpha+1}\leq\overline{X}_{t}-c_{t}<\mu<\overline{X}_{t}+c_{t}\leq\overline{X}_{t}+\frac{\overline{X}_{t}}{\alpha+1}=\left(\frac{\alpha+2}{\alpha+1}\right)\overline{X}_{t}.
\]
\end{proof}
The following capped version of the previous theorem interrupts the process if during the multiplicative mean estimation it is learned that $\mu$ is smaller than some threshold parameter $\theta$.
\begin{cor}[Capped Mean Estimation] 
\label{cor:cmes}
For any threshold parameter $\theta \in [0,1]$, under the same assumptions of Theorem~\ref{thm:mme}, define $\tau_\theta = \min \{ \tau, t_\theta\}$, where
\[
	t_{\theta}
=
	\left\lceil \frac{\alpha+2}{3\theta}\ln\left(\frac{1}{\delta}\right)\left(\sqrt{9\alpha^{2}+114\alpha+192}+3\alpha+19\right)\right\rceil +2=\mathcal{O}\left(\frac{\alpha^{2}}{\theta}\ln\frac{1}{\delta}\right)~.
\]
With probability at least $1-3(T-1)\delta$, 
\begin{enumerate}
\item \label{i:g1} if $\tau_{\theta}=\tau$, then for all $t\in\{2,\ldots,T\}$ such that (\ref{eq:bernoulli_condition-1})
holds, inequalities (\ref{eq:mme}) also hold;
\item \label{i:g2} if $\tau_\theta = t_\theta$, then $\mu \leq \theta$.
\end{enumerate}
\end{cor}
Our Capped Mean Estimation is defined as the Capped Mean Estimation of the demand curve (or one minus the demand curve if $\rho=1$) at a given sequence of 
prices\footnote{This algorithm is only used for prices $x_1,x_2,\l$ such such that $D(x_s)=D(x_t)$ for all $s,t$.} 
$x_1,x_2,\l$, with threshold $\theta\in [0,1]$ (where $1/\theta$ is interpreted as $\iop$ when $\theta=0$), confidence parameter $\delta\in(0,1)$, reverse parameter $\rho$ (that regulates if $D(x_1)$ or $1-D(x_1)$ is being estimated) and $\alpha =1$ (Algorithm~\ref{algo:cme}).
\begin{algorithm2e}[t]
	\caption{Capped Mean Estimation\label{algo:cme}}
	\LinesNumbered
	\SetKwInput{kwInit}{Initialization}
	\KwIn{$x_1,x_2,\l\in[0,1]$, $\theta \in [0,1]$, $\delta \in (0,1)$, $\rho\in \{0,1\}$.}
	\kwInit{set $t\gets 3$ and $\Dhat_s = (1-\rho) \Ind{V_s \ge x_s} + \rho(1-\Ind{V_s \ge x_s})$ for all $s$.}
	offer $x_1$ and $x_2$ once each;

	set 
	$\Dbar \gets \frac{1}{2}\sum_{s=1}^2 \Dhat_s $ 
	and 
	$S^2 \gets \sum_{s=1}^2 \big(\Dhat_s - \Dbar\big)^2$;
	
	\While{$\big[ t \le \lceil 40\ln(1/\delta) / \theta \rceil +2 \big]
		\wedge 
		\big[\Dbar < \sqrt{8S^2\ln(1/\delta)/t} + (14/3) \ln(1/\delta)/(t-1) \big]$}{
	offer price $x_t$ once;

	update 
	$\Dbar \gets \big(\Dbar (t-1) + \Dhat_t\big)/t$, 
	$S^2 \gets \big( S^2(t-2) + (\Dhat_t - \Dbar)^2\big)/(t-1)$, 
	and 
	$t\gets t+1$;
	}
	\lIf{$t > \lceil 40\ln(1/\delta) / \theta \rceil +2$}{return that $\mu\le\theta$}
	\lElse{return $\Dbar/2$}
\end{algorithm2e}
\subsubsection*{Variant: Joint Capped Mean Estimation}
We call $(w,\theta,\delta)$-Joint Capped Mean Estimation a variant of Algorithm~\ref{algo:cme} in which $x_t=w$ for all $t$ and estimations for both $\rho=0$ and $\rho=1$ are carried on at the same time; i.e., where both $\Dbar$ (sample mean for $\rho =0$) and $\Dbar'=1-\Dbar$ (sample mean for $\rho =1$), as well as their respective sample variances $S^2$ and $(S')^2$ are maintained; the condition $\big[\Dbar \le  \sqrt{8S^2\ln(1/\delta)/t} + (14/3) \ln(1/\delta)/(t-1) \big]$ in the \textbf{while} loop is replaced by
\[
	\big( A \vee A' \big)
=
	\left( 
	\left[\Dbar < \sqrt{\frac{8 S^2}{t} \ln \frac{1}{\delta} } + \frac{14}{3(t-1)} \ln \frac{1}{\delta} \right]
\vee
	\left[\Dbar' < \sqrt{\frac{8 (S')^2}{t} \ln \frac{1}{\delta} } + \frac{14}{3(t-1)} \ln \frac{1}{\delta} \right]
	\right)
\]
and at the end, we return $\Dbar/2$ (resp., $\Dbar'/2$) and we say that $D(w)$ (resp., $1-D(w)$) is \emph{well-estimated} if and only if $A$ (resp., $A'$) is false; if $A$ (resp., $A'$) is true we return that $D(w)$ (resp., $1-D(w)$) is at most $\theta$.
\subsubsection*{Variant: Capped Mean Estimation on Noisy Cautious Search}
With a slight abuse of notation, we say that a $(\theta,\delta,\rho)$-Capped Mean Estimation is run on a $(\delta,i,\gamma_i)$-Noisy Cautious Search if $x_1, x_2,\l$ are the prices offered during the first successful macrosteps of a $(\delta,i,\gamma_i)$-Noisy Cautious Search run for $\Theta\big(\frac{1}{D(x_1)}\big)$ macrosteps (resp., $\Theta\big(\frac{1}{1-D(x_1)}\big)$ macrosteps); i.e., while the Noisy Cautious Search proceeds, an increasingly accurate estimate $\phat$ of $D(x_1)$ (resp., $1-D(x_1)$) is maintained at the same time using samples from successful macrosteps; as soon as the stopping criterion for the Capped Mean Estimation is met, the estimation stops while the Noisy Cautious Search proceeds until it reaches $\lceil 6/\phat \rceil$ macrosteps, at which point the whole process ends returning $\phat$ and the price $\vhat_i$ offered during the last succesful Noisy Cautious Search macrostep.
\subsection*{Cautious Mean Estimation}
The main idea of this section is that the problem for $K=2$ is completely solved by determining $v_1$, $v_2$, and $p_2$. This suggests that computing an high-confidence estimate $p_2$ once a value $w\in(v_1,v_2]$ is located might be a good idea. Sadly, it is not. The problem with this approach is that if $p_2$ is very small an arbitrary high regret may be incurred in doing so. On the other hand, the more evidence is gathered that $p_2$ is very small, the less likely it is that $v_2$ is optimal. For these and other more subtle reasons, a great deal of caution is needed in order to obtain estimate of $p_2$ that is just good enough to use. 

The algorithm we present for dealing with these issues is called Cautious Mean Estimation and it receives as an input a price $w\in (v_1,v_2]$ (i.e., that can be used to estimate $p_2$), as well as a confidence parameter $\delta$. The routine begins by determining if $p_1$ and $p_2$ are both bigger than $1/4$ by using a Joint Capped Mean Estimation and invoking Corollary~\ref{cor:cmes}. If this is true, it simply returns the estimates of $p_1$ and $p_2$ to the main routine; otherwise it behaves differently depending on which one is true: $p_2\le 1/4$ or $p_2 \ge 3/4$, which can be checked invoking again Corollary~\ref{cor:cmes}. If $p_2 \le 1/4$, it proceeds in phases. In each phase $s$, it checks if $v_1 \ge 2^{-s}$ by offering $2^{-s}$ a small number of times, in which case it halts returning that $v_1$ is the optimum. If it is not, it determines if $p_1$ and $p_2$ are bigger than $2^{-(k+1)}$ by using one more time Corollary~\ref{cor:cmes}, in which case it returns their estimates to the main routine. If they are not, it moves on to phase $k+1$. If on the other hand $p_2$ was bigger than $3/4$, it performs a Noisy Cautious Search for $v_2$, while at the same time collecting samples to estimate $p_1$, returning estimates $\vhat_2$ and $\phat_1$. Then it first checks if $v_1 \le \vhat_2 (1-\phat_1) - \phat_1$ by posting the latter for $\big\lceil\ln(1/\delta)/\phat_1\big\rceil$ rounds. If the test is positive, it halts returning that $v_2$ is the optimum. Otherwise it returns $\phat_1$ and $\phat_2$ to the main routine.
\begin{algorithm2e}[t]
	\caption{Cautious Mean Estimation\label{algo:cume}}
	\LinesNumbered
	\KwIn{price $w\in[0,1]$, confidence parameter $\delta \in (0,1)$.}
	\label{i:jcme1}run a $(w,2^{-2},\delta)$-Joint Capped Mean Estimation, returning $\phat_1,\phat_2$;

	\uIf(\tcp*[f]{$1/4 \le p_1, p_2 \le 3/4$}){$D(w)$ and $1-D(w)$ are both well-estimated\label{i:if1}}{
		return $\phat_1, \phat_2$;
	}
	\uElseIf(\tcp*[f]{$p_1 > 3/4$}){$1-D(w)$ is well-estimated\label{i:if2}}{
		\For{$s \in \{2,3,\l\}$\label{i:for}}
		{
			offer $2^{-s}$ for $\big\lceil\ln(\delta)/\ln(3/4)\big\rceil$ rounds\label{i:pay1};

			\lIf{all offers are accepted\label{i:br1}}
				{\textbf{break} and return that $v_1$ is optimal}
				\uElse
				{
					continue the Joint Capped Mean Estimation with new parameters $w,2^{-(s+1)},\delta$\label{i:pay2};

					\lIf{$p_1$ and $p_2$ are both well-estimated\label{i:br2}}
						{\textbf{break} and return $\phat_1, \phat_2$}
				}
		}
	}
	\uElseIf(\tcp*[f]{$p_2 > 3/4$}){$D(w)$ is well-estimated\label{i:if3}}{
		run $(0,\delta,1)$-Capped Mean Estimation on $\big(\delta,2,\frac{3}{4}\big)$-Noisy Cautious Search, returning $\phat_1, \vhat_2$\label{i:nca};

		offer $\vhat_2 \wh{q}_2 - \phat_1$ for $\big\lceil\ln(1/\delta)/\phat_1\big\rceil$ rounds, 
		where $\wh{q}_2 \gets 1-\phat_1$\label{i:test-pruned1};

		\lIf{at least one offer is rejected}
			{return that $v_2$ is optimal\label{i:test-pruned2}}

		\lElse
			{return $\phat_1, \phat_2$}
	}
\end{algorithm2e}
\begin{lemma}
\label{lem:caut-m-est}
For all $w\in(v_1,v_2]$ and all $\delta\in (0,1)$, the Cautious Mean Estimation run with parameters $w, \delta$ satisfies the following with probability at least $1-(15T-13)\delta$:
\begin{enumerate}
\item \label{i:c1} if the algorithm returns that 
$v_1$ 
or $v_2$ 
is optimal, then it is correct;
\item \label{i:c2} if the algorithm returns $\phat_1$ and $\phat_2$, then both satisfy $p_i/3 < \phat_i < p_i$;
\item \label{i:c3} the regret of the algorithm it at most $(13)^2\ln(1/\delta)+6$.
\end{enumerate}
\end{lemma}
\begin{proof}
By definition of Joint Capped Mean Estimation, line~\ref{i:jcme1} lasts for at most $\big\lceil 160\ln(1/\delta) \big\rceil+2$ rounds, which upper bounds the regret accrued during those time steps. Denote $G$ the \emph{good} event in which which items \ref{i:g1} and \ref{i:g2} of Corollary~\ref{cor:cmes} hold simultaneously for both the estimate of $p_1$ and $p_2$. To prove the result, we can (and do!) restrict our analysis to \emph{good} outcomes, i.e., outcomes belonging in $G$. Indeed, Corollary~\ref{cor:cmes} implies that one and only one of the three conditions at lines \ref{i:if1}, \ref{i:if2}, and \ref{i:if3} is executed with probability at least $\P(G) \ge 1-6(T-1)\delta$ and we will show that the result holds in all three cases. 

If the condition at line~\ref{i:if1} is true, then the result follows immediately by Corollary~\ref{cor:cmes}. 

Assume now that the condition at line~\ref{i:if2} is true and fix $k\in\N $ such that $2^{-k} \le \max\{v_1, p_2\} \le 2^{-(k-1)}$. Note that if $v_1 \ge p_2$, the loop at line~\ref{i:for} will break with probability at least $1 - \delta$ (by Chernoff inequality) at line \ref{i:br1} as soon as $s=k$; this proves point~\ref{i:c1} for $v_1$. If on the other hand $v_1 < p_2$, the loop will break with probability at least $1-6(T-1)\delta$ (by Corollary~\ref{cor:cmes}) at line \ref{i:br2} as soon as $s=k-1$; this proves point~\ref{i:c2}. In any case, then, at most $k-1$ cycles of the loop are performed with probability at least $1-(6T-5)\delta$. If $s\le k$, line~\ref{i:pay1} is performed at most $k-1$ times and since the cost of sampling is at most $v_1$ (if $v_1$ is optimal) or $p_2$ (if $v_2$ is optimal), than the total regret accrued by executing line~\ref{i:pay1} is at most $(k-1) \big\lceil\ln(\delta)/\ln(3/4)\big\rceil \max\{v_1, p_2\} \le (e \ln 2)^{-1} \big\lceil\ln(\delta)/\ln(3/4)\big\rceil$, where we used $x\log_2 (1/x)\le (e \ln 2)^{-1}$, for all $x>0$. On the other hand, by the end of phase $k$ the Joint Capped Mean Estimation at lines~\ref{i:jcme1},~\ref{i:pay2} has offered $w$ for at most $\big\lceil 2^{k+1} 40 \ln(1/\delta) \big\rceil +2$ accruing at most $40 \ln (1/\delta) +3$ regret. This proves point~\ref{i:c3}.

Finally, consider the case in which the condition at line~\ref{i:if3} is true. The Noisy Cautious Search at line~\ref{i:nca} stops after at most $\big\lceil 40 \ln(1/\delta)/p_1 \big\rceil +2$ rounds, returning $\phat_i \in (p_i/3, p_i)$, with probability at least $1-(7T-6)\delta$ by the fact that it makes a mistake with probability at most $\delta$ and Theorem~\ref{thm:mme}. This proves point~\ref{i:c2}. If $v_2$ is optimal, Lemma~\ref{lm:noisy-cs} shows that the regret of the Noisy Cautious Search is at most $\big( 3 (\ln \ln T) + 8 \ln(1/\delta) \big) \ln (4/3)$ with probability at least $1- T\delta$. If $v_1$ is optimal, the additional regret is at most $\big( \lceil 40 \ln(1/\delta)/p_1 \rceil +2 \big) (v_1 - wp_2) \le 40 \ln (1/\delta)+3$. 

Consider now lines~\ref{i:test-pruned1}-\ref{i:test-pruned2}. Since $p_1/ 3 < \phat_1 < p_1$, then $p_2 < \wh{q}_2 < p_2 + (2/3)p_1$. Furthermore, $v_2 -p_1 \le \vhat_2 \le v_2$ with probability at least $1-T\delta$ by Lemma~\ref{lm:noisy-cs}. If the test at line~\ref{i:test-pruned2} is true, then $v_1 < v_2 p_2$ and $v_2$ is optimal with probability at least $1-\delta$; this proves point~\ref{i:c1} for $v_2$. To compute the regret accumulated at line~\ref{i:test-pruned1}, assume first that $v_1$ is optimal; then necessarily $v_1 \ge \vhat_2 \wh{q}_2 - \phat_1$ and the regret of line~\ref{i:test-pruned1} is at most
$
	(v_1 - \vhat_2 \wh{q}_2 + \phat_1) \big\lceil {3 \ln (1/\delta)} / {p_1} \big\rceil
\le 
	9 \ln (1/\delta)+3.
$
If on the other hand $v_2$ is optimal, then the regret of line~\ref{i:test-pruned1} is at most
$
	(p_2v_2 - \vhat_2 \wh{q}_2 + \phat_1) \big\lceil {3 \ln (1/\delta)} / {p_1} \big\rceil
\le
	6 \ln (1/\delta) + 2.
$
This proves point~\ref{i:c3} and concludes the proof.
\end{proof}
\subsection*{2-UCB}
This subroutine is a slightly modified version of Algorithm~\ref{alg:stoc_gamma}. The only differences are that two feasible intervals are initialized at the beginning, each valuation $v_i$ gets a personalized number of rounds $\big\lceil 8 \ln(\delta) / \ln (1-\gamma_i) \big\rceil$ at line~\ref{state:post-stoc}, and the test at line~\ref{state:if_branch_stoc} need not be executed as it is known in advance that $K=2$. 

The following result is a straightforward adaptation of Theorem~\ref{th:analisys-kval}. As such, the proof is omitted.
\begin{lemma}\label{Lemma:2-UCB}
If $\Delta = |p_2 v_2 - v_1|$, $\gamma_1 \le p_1$, $\gamma_2 \le p_2$, and 2-UCB run with $\delta = T^{-2}$, it incurs a regret
\[
	\co \left( \frac{\ln T}{\Delta} + (\ln T) (\ln \ln T) \left( \frac{\alpha_1}{-\ln(1-\gamma_1)} + \frac{\alpha_2}{-\ln(1-\gamma_2)} \right) \right)~,
\]
where $(\alpha_1, \alpha_2) = (v_1 - v_1p_2, v_1)$ if $v_1$ is optimal, $(\alpha_1, \alpha_2) = (v_2p_2 - v_1p_2, p_2 v_2)$ if $v_2$ is optimal, and the first term is absent if $\Delta = 0$.
\end{lemma}
\subsection*{Proof of Theorem \ref{TH:twoval-prob}}
We finally have all the instruments to prove Theorem~\ref{TH:twoval-prob}, that we restate for completeness.
\begin{algorithm2e}[t]
	\caption{\label{algo:twoval-stoc}}
	\LinesNumbered
	\KwIn{Confidence parameter $\delta \in (0,1)$.}
	run a Binary~Search, returning $[a_1,a_2]$\label{i:bs1}\tcp*{phase 1}
	run a Capped~Mean~Estimation of the demand at $a_2$ with parameter $\theta = a_1$, returning $\widetilde{p}_2$\label{i:post1};

	\lIf{$\widetilde{p}_2 > 0$}{
		set $w \gets a_2$
	}
	\uElse{
		offer price $a_1$ until it is rejected\tcp*{check if $a_1 < v_1 \le v_2 < a_2$}
		set $w \gets a_1$\label{i:ph1-e};
	}
	run a Cautious~Mean~Estimation of the demand at $w$, returning $\phat_1$ and $\phat_2$\label{i:ph2}\tcp*{phase 2}	
	\uIf{the Cautious~Mean~Estimation was halted because $v_1$ or $v_2$ is the obvious optimum}{
		run a Cautious~Search for the optimal valuation with lower bound $1/2$\label{i:ph2-e};
	}
	\lElse(\tcp*[f]{shrink the interval}\label{i:ph3}){%
		run {2-UCB} with parameters $\gamma_1 = \phat_1$ and $\gamma_2 = \phat_2$
	}
\end{algorithm2e}
\begin{theorem}{\ref*{TH:twoval-prob}}
If Algorithm \ref{algo:twoval-stoc} is run on two unknown pairs $(v_1,p_1)$ and $(v_2,p_2)$ with input parameter $\delta = T^{-2}$, then its regret satisfies
\[
	R_T 
= 
	\co \left( \frac{\log T}{\Delta} + (\log T) (\log\log T) \right)~,
\]
where the first term is absent if $\Delta = |p_2 v_2 - v_1|$ is zero.
\end{theorem}
\begin{proof}
Putting together the proofs of all previous lemmas, the probability of making a mistake in at least a test of at least a routine is upper bounded by $\co (T\delta)$. For this reason, we can (an do) assume that no mistakes happen. We divide the proof into three different cases.
\paragraph*{Case 1}
Assume that during phase~1 all offers of $a_2$ are rejected and all offers of $a_1$ are accepted. Consider the following four subcases. 
If $a_1 \le v_1 \le v_2 \le a_2$, the regret is at most $\co(\log T)$.
Assume now that $v_1 \le a_1 \le v_2 \le a_2$. If $v_2$ is optimal, then the regret is at most $\co \big(\log (T)/a_1\big) = \co \big( \log(T) / \Delta \big)$. If $v_1$ is optimal, then the regret is at most $\co \big((v_1 - a_1 p_2)T + v_1 \ln(T) / a_1 \big) = \co \big( a_1 p_1 T + \ln(T) \big)$. Note that this case only happens with probability $p_2^{\co(T- \ln(T)/a_1)}$, which is at least $1/T$ only if $p_1 = \co \big( \frac{\log T}{ T - \log (T) / a_1 }\big)$. Now, if $a_1 = \Omega \big(\log(T)/T\big)$ then the regret is at most $\co (\log T)$; otherwise it is at most  $\co ( \log T )$ because $v_1$ is small.
If $a_1 \le v_1 \le a_2 \le v_2$, the regret is at most $\co \big( (\max\{v_1, p_2 v_2\} - a_1 ) T + \max\{v_1, p_2 v_2\} \log(T)/a_1\big).$ Since all offers of $a_2$ were rejected, Corollary~\ref{cor:cmes} implies that $p_2 \le a_1$, then $p_2 v_2 \le a_1 \le v_1$, hence $v_1$ is optimal. The regret is therefore at most $\co(\log T)$.
Finally, assume that $v_1 \le a_1 \le a_2 \le v_2 $. Combining the same arguments as above, $v_1$ is optimal but $p_1$ is small and the total regret is at most $\co (\log T)$.

\paragraph*{Case 2}
Assume that during phase~1 some offers of $a_2$ are accepted. Corollary~\ref{cor:cmes} implies that the first Capped Mean Estimation lasts at most $\co \big(\log (T) / \max\{a_1,p_2\} \big)$ rounds, hence its regrets is at most $\co ( \log T )$. Lemma~\ref{lem:caut-m-est} implies that the Cautious Mean Estimation has a regret at most $\co(\log T)$. If the cautious mean estimation is halted returning that $v_1$ or $v_2$ is optimal, then Lemma~\ref{lm:noisy-cs} implies that the regret is at most $\co\big((\log T) (\log \log T)\big)$. Assume now that the cautious mean estimation returns $\phat_1,\phat_2$. By construction, if $p_2 \le 1/4$, then necessarily $v_1 \le 2 p_2$, thus $\Delta \leq 2p_2$. On the other end, if $p_2 \ge 3/4$, then necessarily $v_1 \geq p_2 v_2 - 2 p_1$ thus $\Delta \leq 2 p_1$  if $v_2$ is optimal.
Using Lemma \ref{Lemma:2-UCB} and plugging in the above upper bounds gives the result.
\paragraph*{Case 3}
Assume that during phase~1 all offers of $a_2$ are rejected and some offers of $a_1$ are rejected.
The proof of this case is the same as the previous one, except that sampling $a_2$ has an extra regret cost.
If $v_1$ is optimal, then the additional regret is at most
$
\co \big( v_1 \log (T)/a_1 \big) = \co ( \ln T )
$
because $v_1 < a_1$.
Finally, assume that $v_2$ is optimal. If $v_2 \leq a_2$, the additional cost is at most
$
\big( \ln(T)/a_1 \big) p_2 v_2 = \co ( \ln T ).
$
If  $a_2 < v_2$, then $p_2 \leq a_1$ by Corollary~\ref{cor:cmes} and the additional cost is at most
$
	\co \big( (p_2 v_2 - p_2 a_2) \log(T) / a_1 \big) = \co ( \log T ).
$
\end{proof}

\newcommand{\lhat}{\wh{\ell}}
\newcommand{\hloss}{\wh{\ell}}
\newcommand{\hLoss}{\wh{L}}
\newcommand{\tloss}{\wt{\ell}}
\newcommand{\scF}{\mathcal{F}}
\renewcommand{\iff}{\Longleftrightarrow}

\section{Nonstochastic dynamic pricing: some initial results}
\label{s:twoval-adv}
In this section we present some initial results for the nonstochastic setting; namely, when the sequence $V_1,V_2,\dots$ is deterministic rather than stochastic. This setting was studied in \citep{kleinberg2003value} without the restriction that each $V_t$ belongs to a common finite set of valuations. We show an upper bound of $\scO(\sqrt{T})$ on the regret in the simple case when $V_t \in \{v_1,v_2\}$ for all $t$ (with $0 \le v_1 \le v_2 \le 1$) and $v_2$ is known. Note that this is not significantly improvable, as a matching lower bound of $\Omega(\sqrt{T})$ can be proven in the stochastic setting even when $v_1$ and $v_2$ are both known. To see that, consider $v_1 = \frac{1}{2}$ and $v_2 = \frac{3}{4}$ with $D(v_2) = \frac{2}{3} \pm \ve$ for $\ve < \frac{1}{3}$, so that $v_1$ has constant revenue $\frac{1}{2}$, and $v_2$ has expected revenue $\frac{1}{2} \pm \frac{3}{4}\ve$. We can now adapt the argument in the proof of the nonstochastic bandit lower bound of \cite{auer2002nonstochastic} for the two equiprobable scenarios $D(v_2) = \frac{2}{3} + \ve$ and $D(v_2) = \frac{2}{3} - \ve$. This allows us to conclude that in the first $T$ rounds any algorithm suffers regret of order $\ve\,T$ unless $v_2$ is played at least $\ve^{-2}$ times. Choosing $\ve = T^{-1/2}$ gives the desired bound.

The algorithm achieving regret $R_T = \scO\big(\sqrt{T}\big)$ uses the nonstochastic bandit algorithm Exp3~\cite{auer2002nonstochastic} fed with losses $\ell_{t}(x)=1-r_t(x)$. Since
\[
	\sum_{t=1}^{T}\Bigl(r_{t}(x) - r_{t}(X_{t})\Bigr) = \sum_{t=1}^{T}\Bigl(\ell_{t}(X_t) - \ell_{t}(x)\Bigr)
\]
always hold, bounding the regret of Exp3 defined with respect to losses is equivalent to bounding the regret with respect to revenues. However, as only $v_2$ is known, we run Exp3 using two actions: $v_2$ and an action, called $b$, that starts at $v_2$ and converges to $v_1$ during the execution of Exp3. In particular, we decrease $b$ by steps of length $T^{-1/2}$ whenever $b$ is played and rejected. Clearly, $b$ stops moving as soon as $b \in \big(v_1-T^{-1/2},v_1\big]$, which is good enough to bound Exp3's future regret. In order to bound the regret incurred while $b > v_1$ holds, note that as long as $b > v_1$ is true, we have: $\ell_t(b) < \ell_t(v_1)$ when $V_t = v_2$, and $\ell_t(b) > \ell_t(v_1)$ when $V_t = v_1$. The problem is that we can not bound deterministically the smallest time $t$ such that $b \le v_1$, as this depends on the buyers' choices and the algorithm's random sequence of actions. On the other hand, we know that $\ell_t(b) > \ell_t(v_1)$ \emph{and} $X_t = b$ can simultaneously occur at most $\sqrt{T}$ times, because $b$ is decreased by $T^{-1/2}$ when this happens. We can exploit this observation as follows: when Exp3 plays $b$ and does not sell, then we feed the algorithm a reduced loss of zero. This has the effect of underestimating the algorithm loss by an amount which is bounded by the number of times the algorithm got a reduced loss. This effect is only increasing the regret by at most $\sqrt{T}$, since is the largest number of times $b$ can be decreased while being larger than $v_1$.

More formally, our algorithm runs Exp3 on the two prices $b_t$ and $v_2$, where $b_t$ is dynamically adjusted during the execution. Price $b_t$ is used to locate $v_1$ and is initially set to $v_2$. Exp3 is run with reduced losses $\tloss_t$ of the form $\tloss_t(b_t) = \ell_t(b_t)\Ind{V_t \ge b_t}$ and $\tloss_t(v_2) = \ell_t(v_2)$, where $\ell_{t}(x)=1-x\,\Ind{V_t \ge x}$ are the true losses. Note that, whenever $b_t > v_1$, $V_t = v_1$ implies $\tloss_t(b_t) = 0 = \ell_t(v_1)$, and $V_t = v_2$ implies $\tloss_t(b) \le \ell(v_1)$. Since $b_t$ is random (it depends on $X_1,\dots,X_{t-1}$), $\tloss_t(b_t)$ and $\ell_t(b_t)$ are also random. Technically, this corresponds to running Exp3 with a nonoblivious adversary ---see, e.g., \cite[Remark~4.1]{cesa2006prediction}. However, the regret bounds of Exp3 hold unchanged even for nonoblivious adversaries.

During the execution of Exp3, $b_t$ is adjusted according to the following rule: if $b_t$ is posted at time $t$ and $V_t < b_t$, then $b_{t+1} = b_t-T^{-1/2}$. For all $t$ let $V_t\in\{v_{1},v_{2}\}$ be the value played by the adversary at time $t$ and $X_t \in \{b_t,v_2\}$ be the price posted by the algorithm.
\begin{theorem}
The regret of the above algorithm satisfies
$
R_{T} \le 2\sqrt{T} + \sqrt{4T\ln 2}
$.
\end{theorem}
\begin{proof}
Since $b_1 \le v_2 < 1$, and because $X_t = b_t$ and $\Ind{V_t < b_t}$ imply $b_{t+1} = b_t-T^{-1/2}$, we have that
\[
	\sum_{t=1}^T \Ind{X_t = b_t} \Ind{V_t < b_t} < \sqrt{T}~.
\]
Therefore, the true total loss of Exp3 deterministically relates to its reduced loss as follows,
\begin{align}
\label{eq:tf}
	\sum_{t=1}^T \tloss_t(X_t) = \sum_{t=1}^T \ell_t(X_t) - \sum_{t=1}^T \ell_t(b_t) \Ind{X_t = b_t} \Ind{V_t < b_t}
\ge
	\sum_{t=1}^T \ell_t(X_t) - \sqrt{T}~.
\end{align}
If $b_t\le v_1$ for some $t$, then $b_t$ is always accepted. Hence it is never decreased further, which in turn implies that $b_t > v_1-T^{-1/2}$ holds for all $t$. So we have that $\tloss_t(b_t) \le \ell_t(v_1) + T^{-1/2}$. Recalling that $\tloss_t(v_2) = \ell_t(v_2)$ and using Exp3 regret bound (see, e.g., \cite[Theorem~3.1]{bubeck2012regret}) applied to the nonoblivious reduced losses $\tloss_t$, we obtain
\begin{align*}
	\E\left[\sum_{t=1}^T \ell_t(X_t) \right] - \sqrt{T}
&\le
	\tag{using~(\ref{eq:tf})}
	\E\left[\sum_{t=1}^T \tloss_t(X_t) \right]
\\&\le
	\min\left\{\sum_{t=1}^T \tloss_t(b_t), \sum_{t=1}^T \tloss_t(v_2)\right\} + \sqrt{4T\ln 2}
\\&\le
	\min\left\{\sum_{t=1}^T \Bigl(\ell_t(v_1) + T^{-1/2}\Bigr), \sum_{t=1}^T \ell_t(v_2)\right\} + \sqrt{4T\ln 2}~.
\end{align*}
Therefore, we get
\begin{align*}
	\E\left[\sum_{t=1}^T \ell_t(X_t)\right]
\le
	\min_{v \in \{v_1,v_2\}} \sum_{t=1}^T \ell_t(v) + 2\sqrt{T} + \sqrt{4T\ln 2}
\end{align*}
concluding the proof.
\end{proof}

\bibliography{ncb}

\begin{thebibliography}{26}
\providecommand{\natexlab}[1]{#1}
\providecommand{\url}[1]{\texttt{#1}}
\expandafter\ifx\csname urlstyle\endcsname\relax
  \providecommand{\doi}[1]{doi: #1}\else
  \providecommand{\doi}{doi: \begingroup \urlstyle{rm}\Url}\fi

\bibitem[Agrawal and Devanur(2014)]{agrawal2014bandits}
S.~Agrawal and N.~R. Devanur.
\newblock Bandits with concave rewards and convex knapsacks.
\newblock In \emph{Proceedings of the fifteenth ACM conference on Economics and
  computation}, pages 989--1006. ACM, 2014.

\bibitem[Amin et~al.(2013)Amin, Rostamizadeh, and Syed]{amin2013learning}
K.~Amin, A.~Rostamizadeh, and U.~Syed.
\newblock Learning prices for repeated auctions with strategic buyers.
\newblock In \emph{Advances in Neural Information Processing Systems}, pages
  1169--1177, 2013.

\bibitem[Auer et~al.(2002{\natexlab{a}})Auer, Cesa-Bianchi, and
  Fischer]{auer2002finite}
P.~Auer, N.~Cesa-Bianchi, and P.~Fischer.
\newblock Finite-time analysis of the multiarmed bandit problem.
\newblock \emph{Machine learning}, 47\penalty0 (2-3):\penalty0 235--256,
  2002{\natexlab{a}}.

\bibitem[Auer et~al.(2002{\natexlab{b}})Auer, Cesa-Bianchi, Freund, and
  Schapire]{auer2002nonstochastic}
P.~Auer, N.~Cesa-Bianchi, Y.~Freund, and R.~Schapire.
\newblock The nonstochastic multiarmed bandit problem.
\newblock \emph{SIAM J. on Computing}, 32\penalty0 (1):\penalty0 48--77,
  2002{\natexlab{b}}.

\bibitem[Babaioff et~al.(2015)Babaioff, Dughmi, Kleinberg, and
  Slivkins]{babaioff2015dynamic}
M.~Babaioff, S.~Dughmi, R.~Kleinberg, and A.~Slivkins.
\newblock Dynamic pricing with limited supply.
\newblock \emph{ACM Transactions on Economics and Computation (TEAC)},
  3\penalty0 (1):\penalty0 4, 2015.

\bibitem[Badanidiyuru et~al.(2013)Badanidiyuru, Kleinberg, and
  Slivkins]{badanidiyuru2013bandits}
A.~Badanidiyuru, R.~Kleinberg, and A.~Slivkins.
\newblock Bandits with knapsacks.
\newblock In \emph{Foundations of Computer Science (FOCS), 2013 IEEE 54th
  Annual Symposium on}, pages 207--216. IEEE, 2013.

\bibitem[Berthet and Perchet(2017)]{berthet2017fast}
Q.~Berthet and V.~Perchet.
\newblock Fast rates for bandit optimization with upper-confidence frank-wolfe.
\newblock In \emph{Advances in Neural Information Processing Systems}, pages
  2222--2231, 2017.

\bibitem[Blum and Hartline(2005)]{blum2005near}
A.~Blum and J.~D. Hartline.
\newblock Near-optimal online auctions.
\newblock In \emph{Proceedings of the sixteenth annual ACM-SIAM symposium on
  Discrete algorithms}, pages 1156--1163. Society for Industrial and Applied
  Mathematics, 2005.

\bibitem[Blum et~al.(2004)Blum, Kumar, Rudra, and Wu]{blum2004online}
A.~Blum, V.~Kumar, A.~Rudra, and F.~Wu.
\newblock Online learning in online auctions.
\newblock \emph{Theoretical Computer Science}, 324\penalty0 (2-3):\penalty0
  137--146, 2004.

\bibitem[Broder and Rusmevichientong(2012)]{broder2012dynamic}
J.~Broder and P.~Rusmevichientong.
\newblock Dynamic pricing under a general parametric choice model.
\newblock \emph{Operations Research}, 60\penalty0 (4):\penalty0 965--980, 2012.

\bibitem[Bubeck and Cesa-Bianchi(2012)]{bubeck2012regret}
S.~Bubeck and N.~Cesa-Bianchi.
\newblock Regret analysis of stochastic and nonstochastic multi-armed bandit
  problems.
\newblock \emph{Foundations and Trends{\textregistered} in Machine Learning},
  5\penalty0 (1):\penalty0 1--122, 2012.

\bibitem[Bubeck et~al.(2013)Bubeck, Perchet, and Rigollet]{bubeck2013bounded}
S.~Bubeck, V.~Perchet, and P.~Rigollet.
\newblock Bounded regret in stochastic multi-armed bandits.
\newblock In \emph{Conference on Learning Theory}, pages 122--134, 2013.

\bibitem[Bubeck et~al.(2017)Bubeck, Devanur, Huang, and
  Niazadeh]{bubeck2017online}
S.~Bubeck, N.~R. Devanur, Z.~Huang, and R.~Niazadeh.
\newblock Online auctions and multi-scale online learning.
\newblock In \emph{Proceedings of the 2017 ACM Conference on Economics and
  Computation}, pages 497--514. ACM, 2017.

\bibitem[Cesa-Bianchi and Lugosi(2006)]{cesa2006prediction}
N.~Cesa-Bianchi and G.~Lugosi.
\newblock \emph{Prediction, learning, and games}.
\newblock Cambridge University Press, 2006.

\bibitem[Cotarmanac'h(2017)]{Cot2017}
A.~Cotarmanac'h.
\newblock Auction mechanics: A buyer's perspective.
\newblock Blogpost, 2017.
\newblock URL \url{https://goo.gl/7Nymnt}.

\bibitem[den Boer and Keskin(2017)]{den2017discontinuous}
A.~den Boer and N.~B. Keskin.
\newblock Discontinuous demand functions: Estimation and pricing.
\newblock 2017.

\bibitem[den Boer(2015)]{den2015dynamic}
A.~V. den Boer.
\newblock Dynamic pricing and learning: historical origins, current research,
  and new directions.
\newblock \emph{Surveys in operations research and management science},
  20\penalty0 (1):\penalty0 1--18, 2015.

\bibitem[Devanur et~al.(2014)Devanur, Peres, and Sivan]{devanur2014perfect}
N.~R. Devanur, Y.~Peres, and B.~Sivan.
\newblock Perfect {B}ayesian equilibria in repeated sales.
\newblock In \emph{Proceedings of the twenty-sixth annual ACM-SIAM symposium on
  Discrete algorithms}, pages 983--1002. SIAM, 2014.

\bibitem[Karp and Kleinberg(2007)]{karp2007noisy}
R.~M. Karp and R.~Kleinberg.
\newblock Noisy binary search and its applications.
\newblock In \emph{Proceedings of the eighteenth annual ACM-SIAM symposium on
  Discrete algorithms}, pages 881--890. Society for Industrial and Applied
  Mathematics, 2007.

\bibitem[Kleinberg and Leighton(2003)]{kleinberg2003value}
R.~Kleinberg and T.~Leighton.
\newblock The value of knowing a demand curve: {B}ounds on regret for online
  posted-price auctions.
\newblock In \emph{Proceedings of the 44th Annual IEEE Symposium on the
  Foundations of Computer Science}, pages 594--605. IEEE, 2003.

\bibitem[Maurer and Pontil(2009)]{MaurerPontil09}
A.~Maurer and M.~Pontil.
\newblock Empirical bernstein bounds and sample-variance penalization.
\newblock In \emph{Conference on Learning Theory}, pages 1--9, 2009.

\bibitem[McAfee(2011)]{mcafee2011design}
R.~P. McAfee.
\newblock The design of advertising exchanges.
\newblock \emph{Review of Industrial Organization}, 39\penalty0 (3):\penalty0
  169--185, 2011.

\bibitem[Rothschild(1974)]{rothschild1974two}
M.~Rothschild.
\newblock A two-armed bandit theory of market pricing.
\newblock \emph{Journal of Economic Theory}, 9\penalty0 (2):\penalty0 185--202,
  1974.

\bibitem[Slivkins and Zeevi(2015)]{SZ15}
A.~Slivkins and A.~Zeevi.
\newblock {Dynamic Pricing Under Model Uncertainty}.
\newblock Tutorial given at the 16th ACM Conference on Economics and
  Computation, 2015.

\bibitem[Wedel and Kamakura(2012)]{wedel2012market}
M.~Wedel and W.~Kamakura.
\newblock \emph{Market Segmentation: Conceptual and Methodological
  Foundations}.
\newblock International Series in Quantitative Marketing. Springer US, 2012.

\bibitem[Weed et~al.(2016)Weed, Perchet, and Rigollet]{weed2016online}
J.~Weed, V.~Perchet, and P.~Rigollet.
\newblock Online learning in repeated auctions.
\newblock In \emph{Conference on Learning Theory}, pages 1562--1583, 2016.

\end{thebibliography}
\bibliographystyle{abbrvnat}

\end{document}